\documentclass[twoside]{article}

\usepackage[accepted]{aistats2024}
\usepackage[utf8]{inputenc} 
\usepackage[T1]{fontenc}    
\usepackage{hyperref}       
\usepackage{url}            
\usepackage{booktabs}       
\usepackage{amsmath,amsfonts}       
\usepackage{nicefrac}       
\usepackage{microtype}      
\usepackage{xcolor}         
\usepackage{todonotes}
\usepackage{xr}

\usepackage[round]{natbib}

\newtheorem{assumptions}{Assumptions}[section]
\newtheorem{theorem}{Theorem}[section]
\newtheorem{lemma}[theorem]{Lemma}

\newtheorem{definition}[theorem]{Definition}

\newenvironment{proof}{\begin{trivlist}
    \item[\hskip\labelsep{\bf Proof.}]}{$\hfill\Box$\end{trivlist}}

\newtheorem{example}[theorem]{Example}

\numberwithin{equation}{section}
\numberwithin{figure}{section}
\numberwithin{table}{section}

\newcommand{\abs}[1]{\left|#1\right|}
\newcommand{\norm}[1]{\left\|#1\right\|}

\newcommand{\ip}[1]{\left\langle#1\right\rangle}

 \newcommand{\hPhi}{{\bar{\Phi}}}
\newcommand{\R}{\mathbb{R}}

\newcommand{\E}{\mathbb{E}}
\newcommand{\e}{\varepsilon}
\newcommand{\cW}{{{\mathcal W}}}

\newcommand{\tet}{\tilde{\eta} }

\renewcommand{\hat}{\widehat}
\newcommand{\ii}{\boldsymbol{i}}

\begin{document}

%

%

\runningauthor{Kexin Jin, Chenguang Liu, Jonas Latz}

\twocolumn[

\aistatstitle{Subsampling Error in Stochastic Gradient Langevin Diffusions}

\aistatsauthor{  Kexin Jin$^{*}$ \And Chenguang Liu$^{*}$ \And  Jonas Latz$^{\dag}$ }

\aistatsaddress{ Princeton University \And  Delft University of Technology \And University of Manchester } ]

\begin{abstract}
The Stochastic Gradient Langevin Dynamics (SGLD) are popularly used to approximate Bayesian posterior distributions in statistical learning procedures with large-scale data. As opposed to many usual Markov chain Monte Carlo (MCMC) algorithms, SGLD is not stationary with respect to the posterior distribution;  two sources of error appear: The first error is introduced by an Euler--Maruyama discretisation of a Langevin diffusion process, the second error comes from the data subsampling that enables its use in large-scale data settings. In this work, we consider an idealised version of SGLD to analyse the method's pure subsampling error that we then see as a best-case error for diffusion-based subsampling MCMC methods. Indeed, we introduce and study the  Stochastic Gradient Langevin Diffusion (SGLDiff), a continuous-time Markov process that follows the Langevin diffusion corresponding to a data subset and switches this data subset after exponential waiting times. There, we show the exponential ergodicity of SLGDiff and that the Wasserstein distance between the posterior and the limiting distribution of SGLDiff is bounded above by a fractional power of the mean waiting time.  We bring our results into context with other analyses of SGLD. 
\end{abstract}

\section{INTRODUCTION AND MAIN RESULTS}

Bayesian machine learning allows the applicant not only to train a model, but also to accurately describe the uncertainty that remains in the model after incorporating the training data. Bayesian approaches are naturally used in conjugate settings, e.g., Gaussian process regression or naive Bayes \citep{bishop} or when appropriate approximations are available, e.g., Variational Bayes \citep{variational}. In other situations, none of this is possible and the Bayesian posterior distribution of the trained model needs to be approximated with a Monte Carlo scheme, such as Markov chain Monte Carlo (MCMC) \citep{Neal1996}. Due to the large amount of available training data and the large computational cost of model/derivative evaluations in, e.g.,  Bayesian deep learning problems, accurate MCMC techniques (e.g. MALA, \citealp{Roberts}) are usually inapplicable. Instead, approximate MCMC techniques, such as the Stochastic Gradient Langevin Dynamics (SGLD) \citet{Teh2011} and its variants are popularly employed. Those methods combine the unadjusted Langevin algorithm (ULA) with \emph{(data) subsampling} as it would be usual in stochastic-gradient-descent-type optimisation algorithms.

In this work, we analyse the error that arises from data subsampling in Langevin-based MCMC algorithms in an idealised dynamical system that we refer to as \emph{Stochastic Gradient Langevin Diffusion} (SGLDiff). Hence, we do not propose a new algorithm, but rather deduct in our continuous-time analysis how well Langevin-based MCMC methods can perform under data subsampling when perfectly sampling the underlying dynamics. {Hence, we focus on the intrinsic error that data subsampling has on a Langevin-based MCMC method -- independently of how time stepping is used to derive the MCMC sampler from the underlying Langevin dynamics. {Interestingly, our best-case error analysis shows for this basic form of SGLDiff a behaviour very similar to the discretised algorithm. Comparisons of discrete-in-time and continuous-in-time dynamical systems are not always conclusive.  However, this result may indicate that the Euler--Maruyama discretisation that is implicitly used to obtain SGLD from SGLDiff, is appropriate.}

\begin{figure*}
    \centering
    \vspace{.3in}
   \includegraphics[scale=0.43,trim={3cm 0.6cm 3cm 1cm},clip]{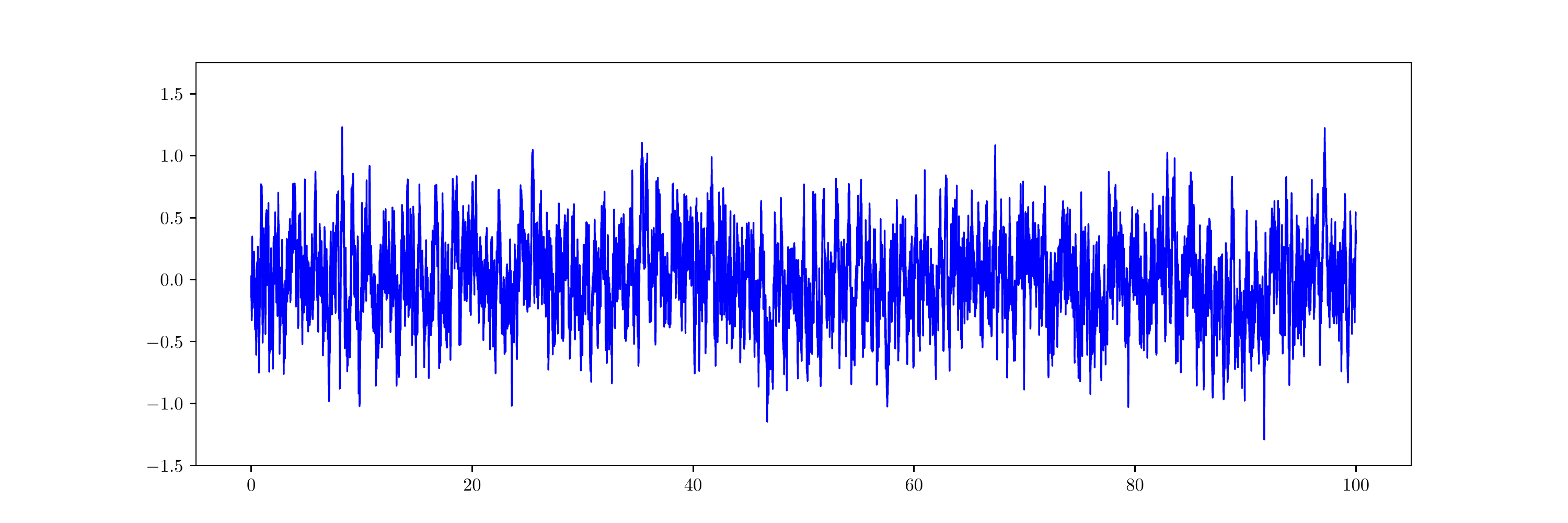} \includegraphics[scale=0.43,trim={0 0.6cm 0 1cm},clip]{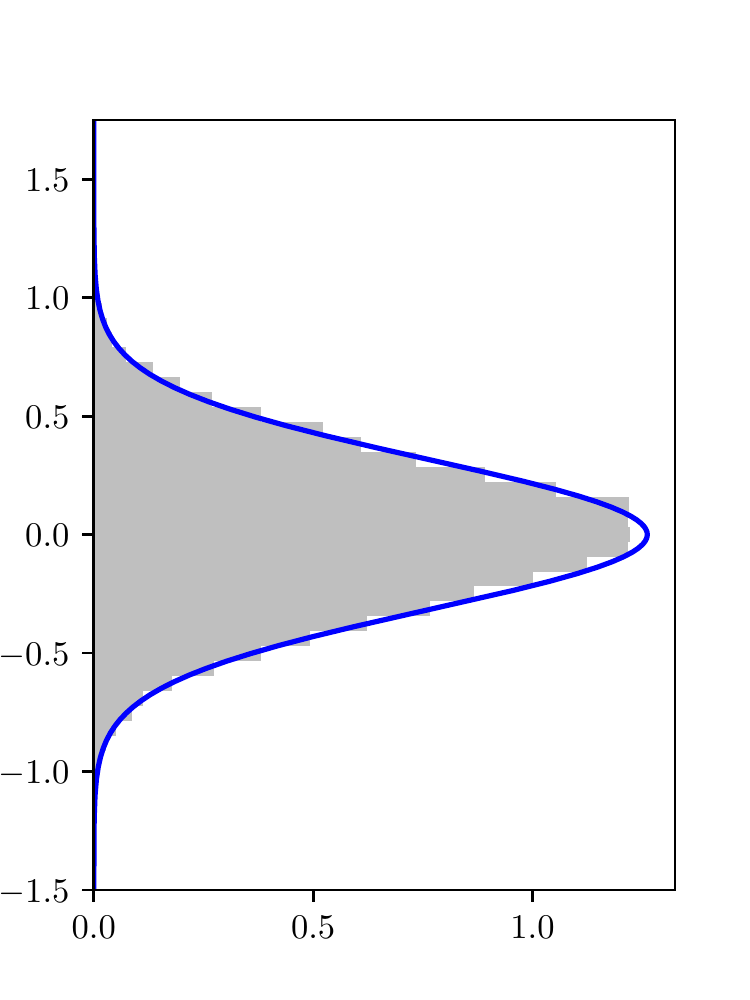}
   \vspace{.3in}
    \caption{Left: Sample path of the diffusion process $\mathrm{d}\zeta_t = - 10 \zeta_t \mathrm{d}t + \sqrt{2}\mathrm{d}W_t$. Right: Its associated stationary density $\mathrm{N}(0,0.1)$ and the histogram of the sample path.}
    \label{fig_diff}
\end{figure*}

\begin{figure*}
    \centering
    \vspace{.3in}
    \includegraphics[scale=0.43,trim={3cm 0.6cm 3cm 1cm},clip]{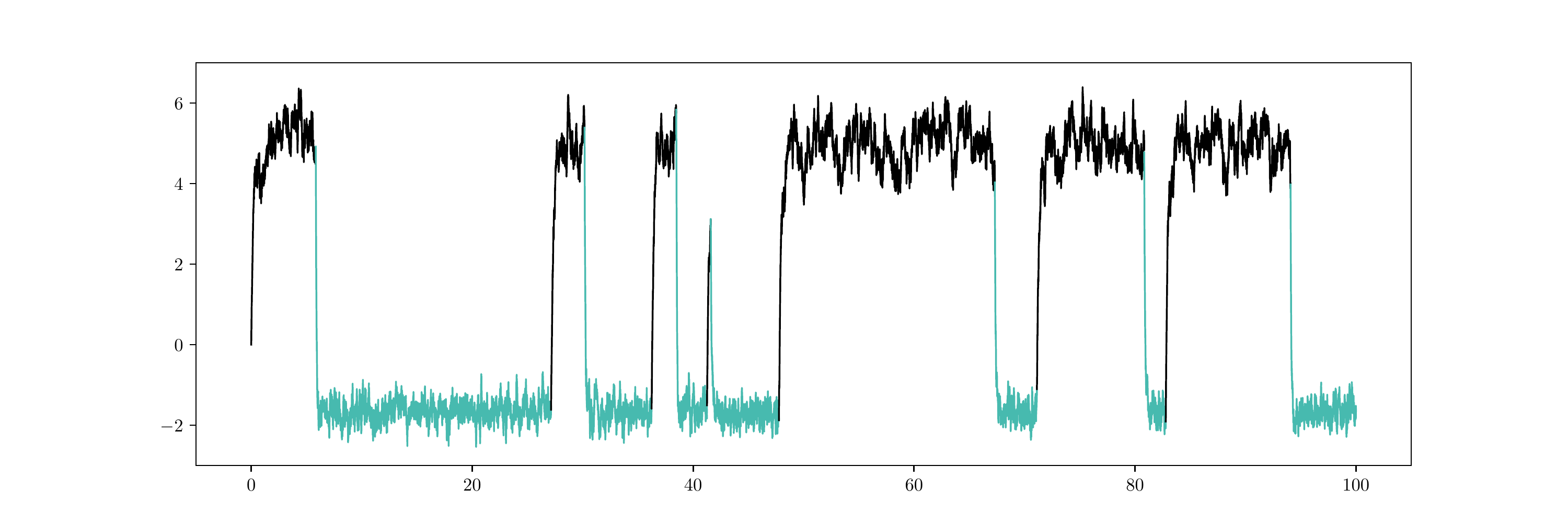} \includegraphics[scale=0.43,trim={0 0.6cm 0 1cm},clip]{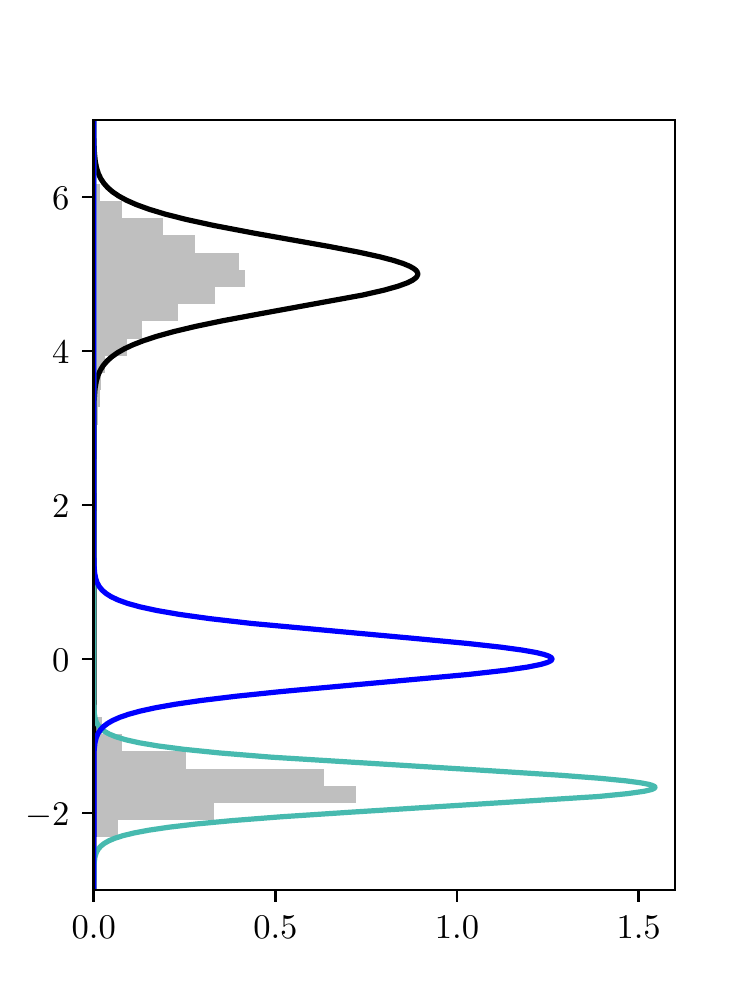}

           \includegraphics[scale=0.43,trim={3cm 0.6cm 3cm 1cm},clip]{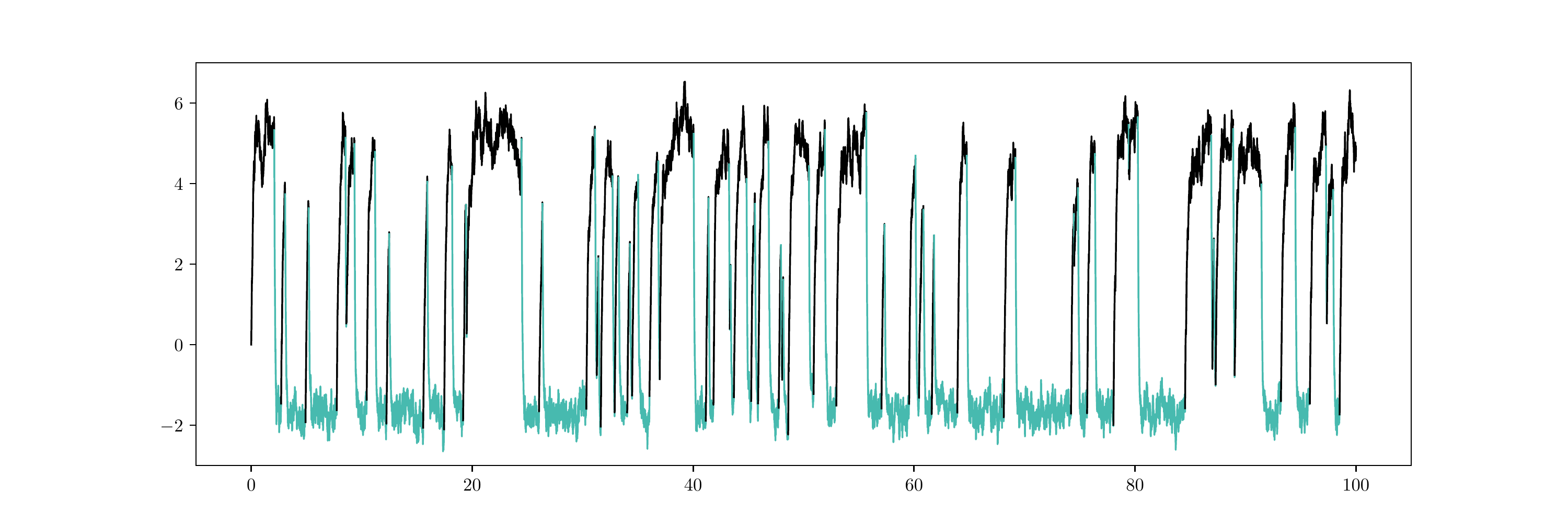} \includegraphics[scale=0.43,trim={0 0.6cm 0 1cm},clip]{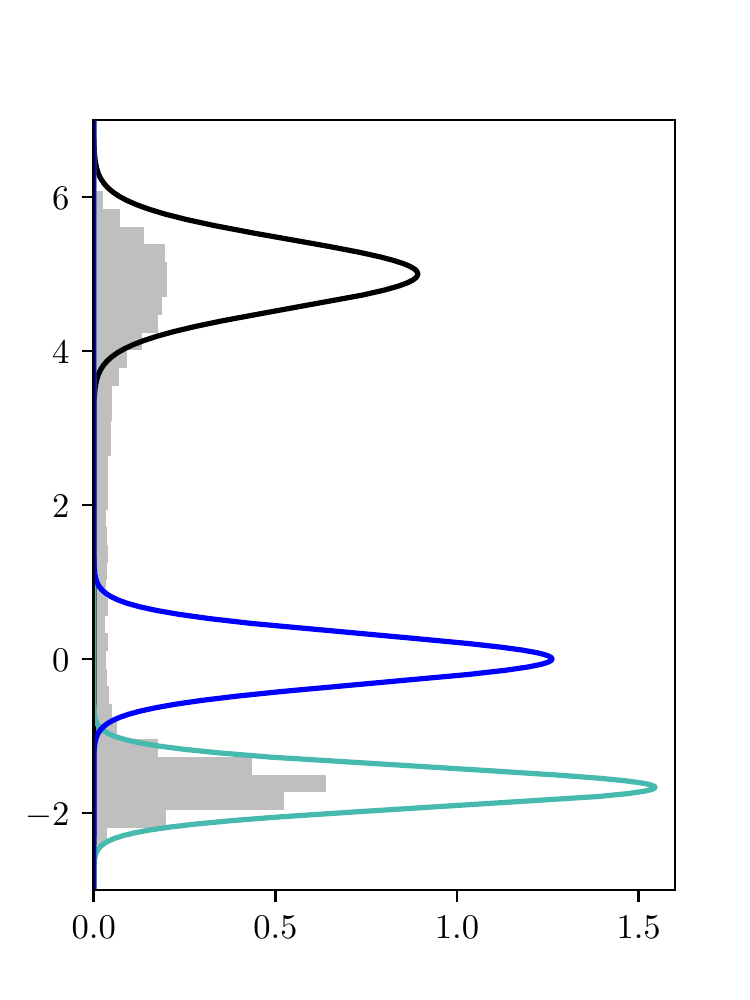}

       \includegraphics[scale=0.43,trim={3cm 0.6cm 3cm 1cm},clip]{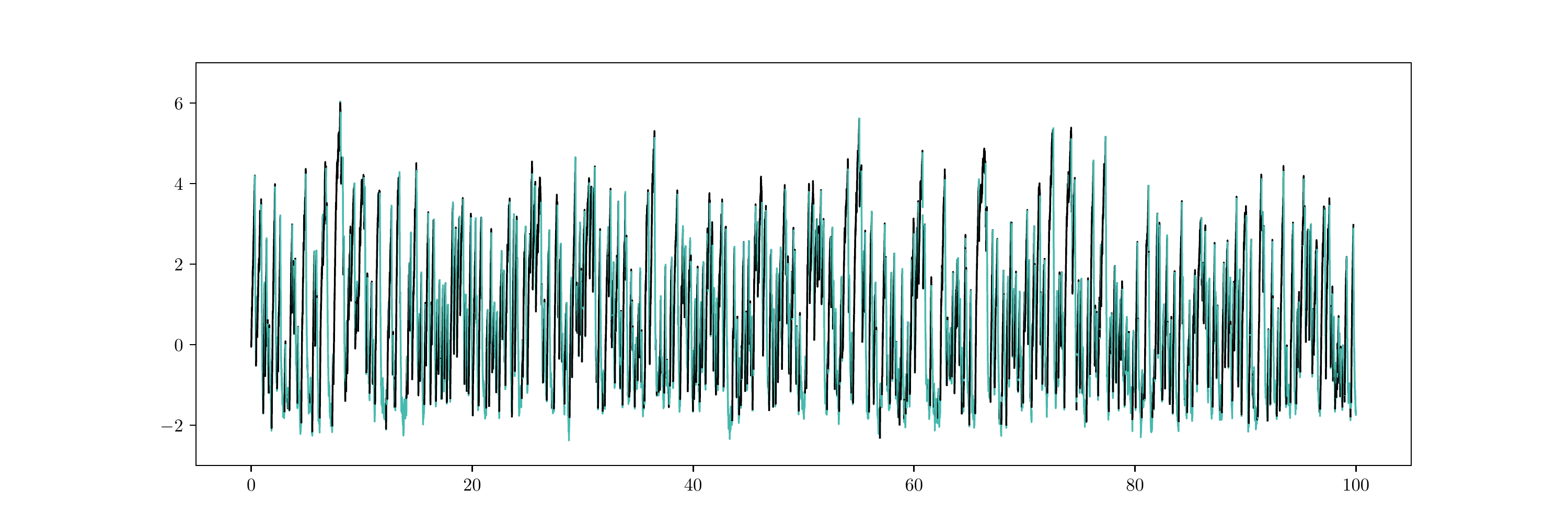} \includegraphics[scale=0.43,trim={0 0.6cm 0 1cm},clip]{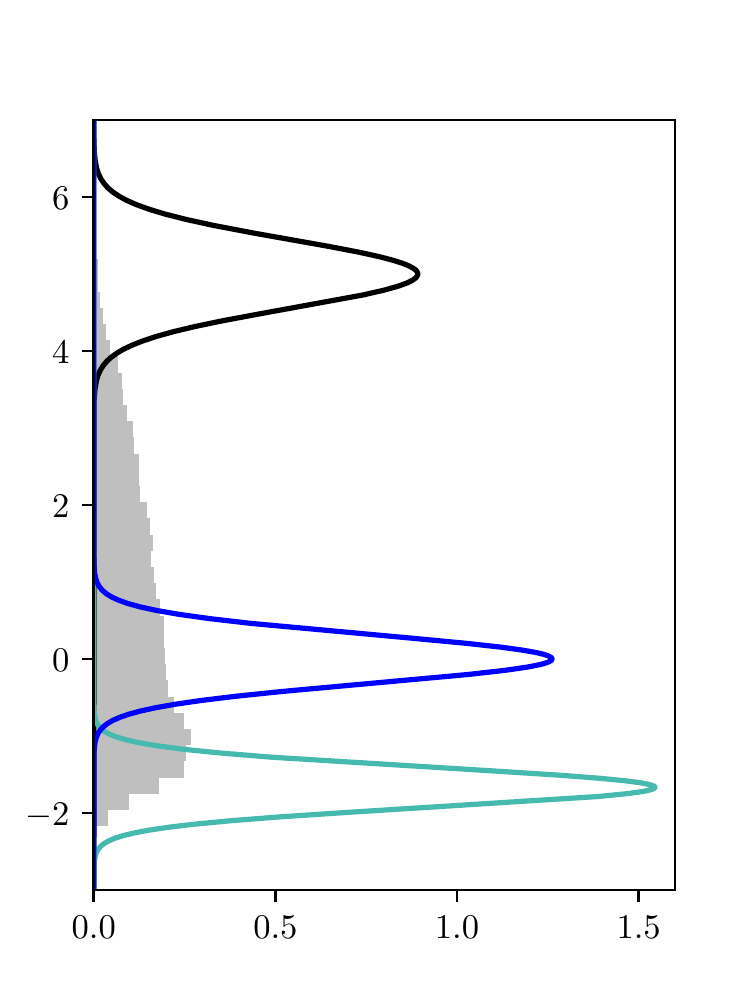}

              \includegraphics[scale=0.43,trim={3cm 0.6cm 3cm 1cm},clip]{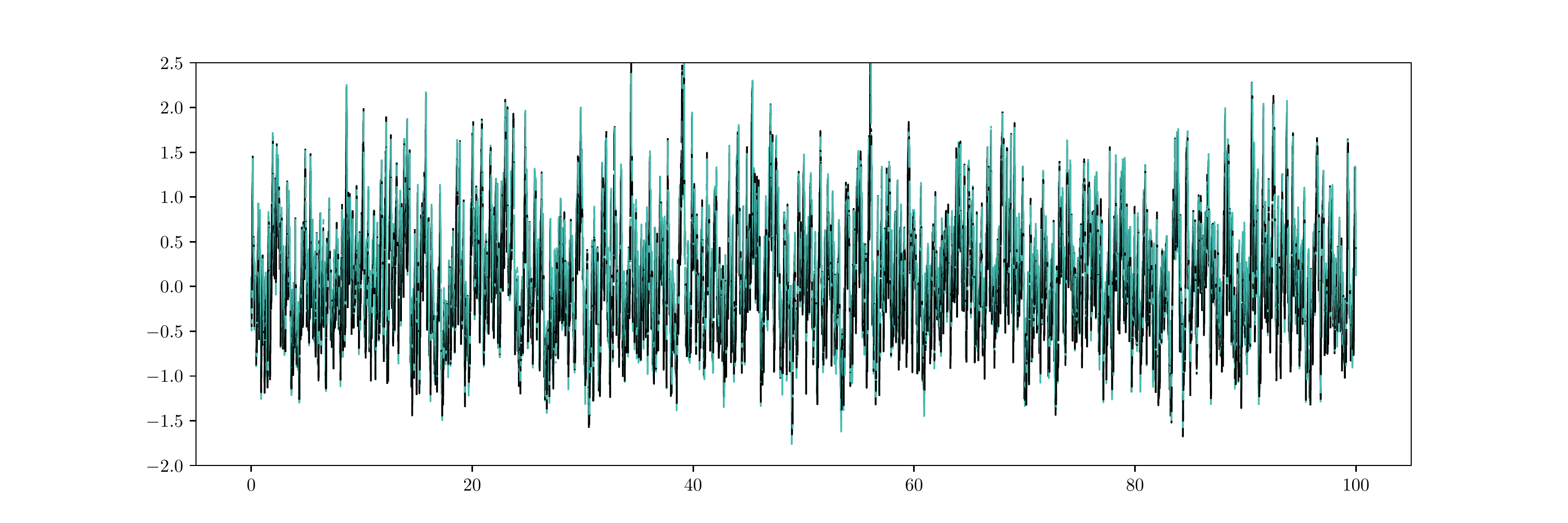} \includegraphics[scale=0.43,trim={0 0.6cm 0 1cm},clip]{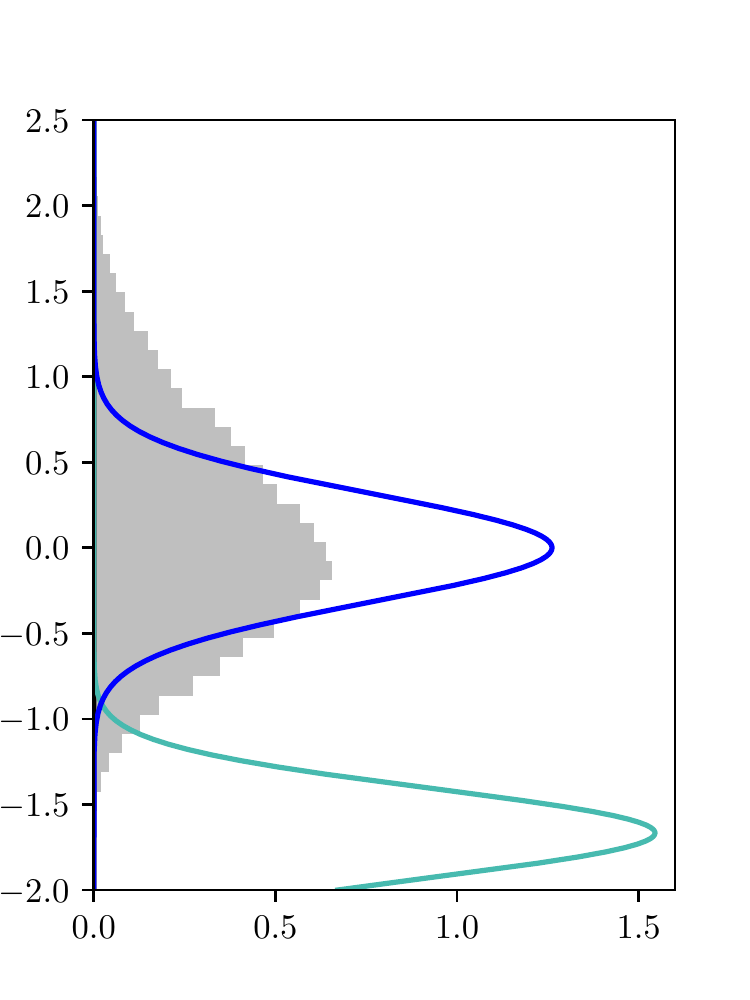}

              \includegraphics[scale=0.43,trim={3cm 0.6cm 3cm 1cm},clip]{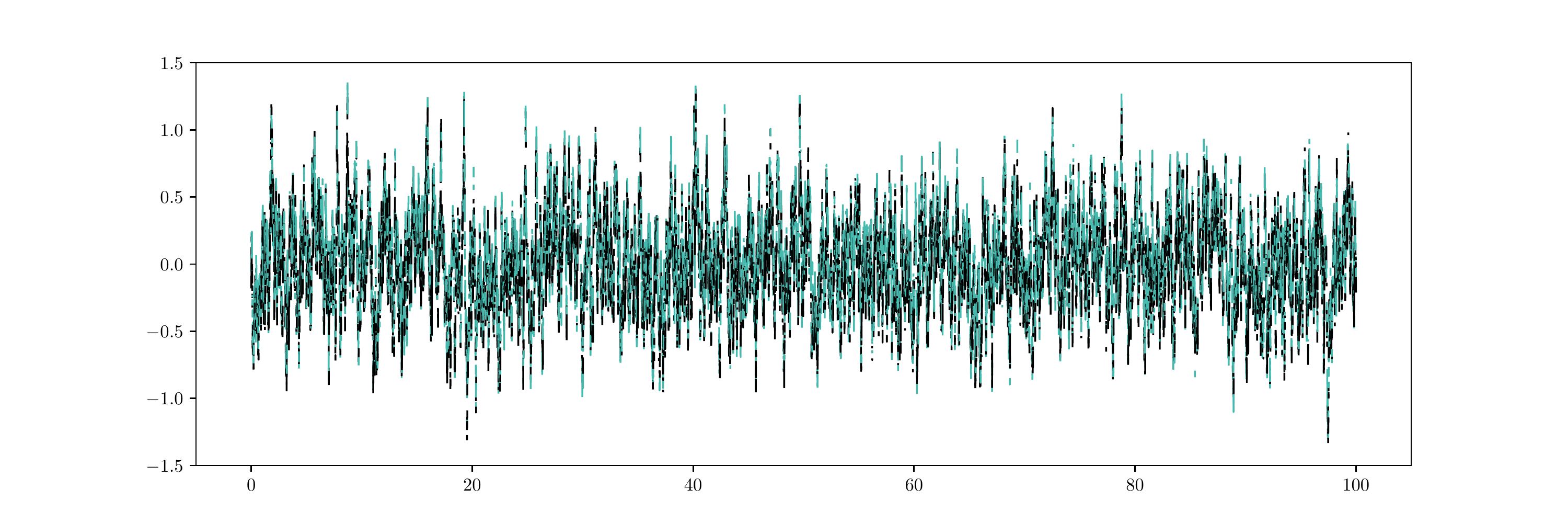} \includegraphics[scale=0.43,trim={0 0.6cm 0 1cm},clip]{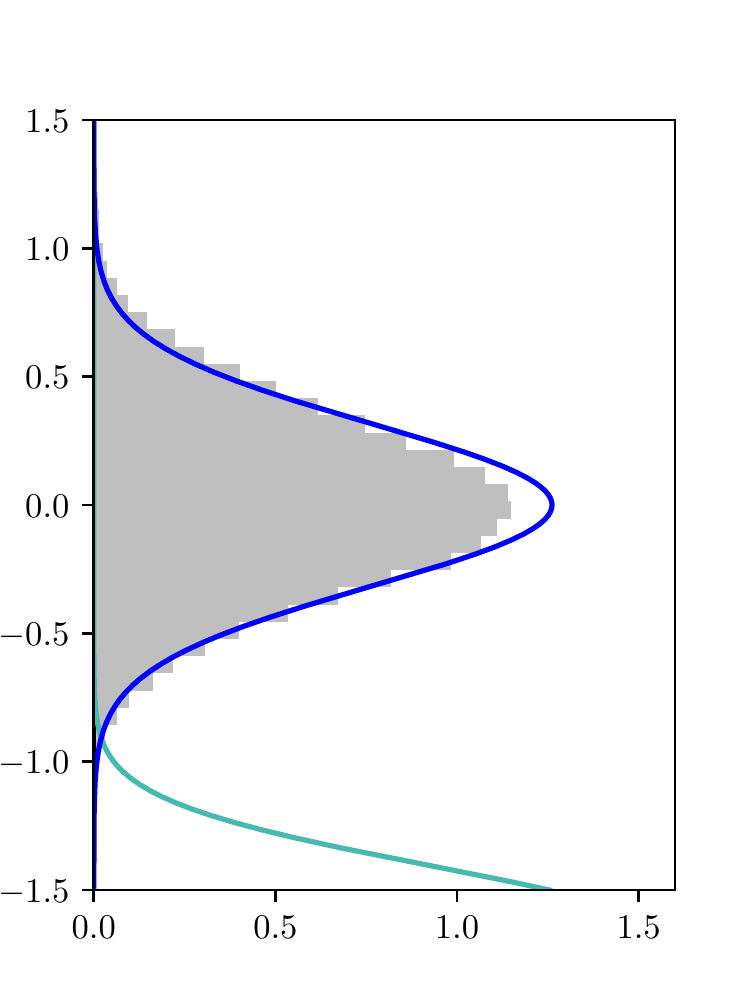}
        \vspace{.3in}
    \caption{Left column: Sample paths of SGLDiffs with $N = 2$ given by $\mathrm{d}\theta_t = a_{\ii(t)}(b_{\ii(t)} - \theta_t) \mathrm{d}t + \sqrt{2}\mathrm{d}W_t$, with $a := (5, 15)$ and $b := (5, -5/3)$, that approximate SDE {$\mathrm{d}\zeta_t = - 10 \zeta_t \mathrm{d}t + \sqrt{2}\mathrm{d}W_t$} and {its stationary} distribution {$\mathrm{N}(0,0.1)$} given in Figure~\ref{fig_diff}, with $\eta = 10^1, 10^0, 10^{-1}, 10^{-2}, 10^{-3}$ (top to bottom). We show the path of $(\theta_t)_{t \geq 0}$ in black whenever $\ii(t) \equiv 1$ and in teal if $\ii(t) \equiv 2$. Right column: Stationary densities of subsampled process (e.g., with fixed $\ii$) in black and teal, respectively, the density of $\mathrm{N}(0,0.1)$ in blue, and the histogram of the sample path in gray. {Note that the scaling of the y-axis changes throughout the plots.}}
    \label{fig:SGLDiff}
\end{figure*}

\section{PROBLEM SETTING}

Throughout this work, we aim to approximate a (target) probability distribution $\mu$ on a space $X := \R^d$ that we equip with the Euclidean norm $\|\cdot \|$ and its associated Borel-$\sigma$-algebra $\mathcal{B}(X)$. We assume that $\mu$ is given by $$\mu(\mathrm{d}\theta) = \frac{1}{Z} \exp\left(- \hPhi(\theta)\right)\mathrm{d}\theta,  $$
where $\hPhi := \frac{1}{N}\sum_{i=1}^N \Phi_i$ is the arithmetic mean of some functions  $\Phi_i: X \rightarrow \mathbb{R}$ that are bounded below, continuously differentiable, and indexed by $i \in I :=\{1,\ldots,N\}$, and
$$
Z := \int_X \exp\left(- \hPhi(\theta')\right)\mathrm{d}\theta' \in (0, \infty)
$$
is the normalising constant.
In a Bayesian learning or inference problem, $\mu$ should be thought off as the posterior distribution. In this case, the function $\Phi_i$ then refers to the regularised data misfit or the negative log-posterior with respect to the data subset with index $i \in I$. Outside of learning and inference, probability distributions of this form also arise in statistical physics.

We use a Monte Carlo approach to approximate $\mu$, e.g., we generate random samples  and then approximate $\mu$ by the associated empirical measure. Here, we rely on MCMC techniques that generate a Markov process that is ergodic and stationary with respect to $\mu$, e.g., the sample trajectory can be used to approximate integrals with respect to $\mu$. Under assumptions on $(\Phi_i)_{i \in I}$, an example for such a Markov process is the solution $(\zeta_t)_{t\geq 0}$ of the following \emph{(overdamped) Langevin dynamic}:
\begin{align}\label{eq: limit}
    \mathrm{d}\zeta_t = -  \nabla\hPhi(\zeta_t) \mathrm{d}t + \sqrt{2} \mathrm{d}W_t,
\end{align}
where $(W_t)_{t \geq 0}$ is a Brownian motion on $X$. We show an example where the Langevin diffusion is used to approximate a Gaussian distribution in Figure~\ref{fig_diff}. In practice, such a Langevin diffusion is used as an inaccurate MCMC algorithm through an Euler--Maruyama discretisation. Indeed, this is the \emph{unadjusted Langevin algorithm}, where the Markov chain $(\hat{\zeta}_{k})_{k=1}^\infty$ is generated by
\begin{equation} \label{eq:ula}
    \hat{\zeta}_{k+1} \leftarrow \hat{\zeta}_{k} -  \eta \nabla\hPhi(\hat{\zeta}_{k}) + \sqrt{2\eta} \xi_k,
\end{equation}
where $\eta > 0$ is the \emph{learning rate} (or step size) and $(\xi_{k})_{k =1}^\infty$ is a sequence of independent and identically distributed (iid) standard Gaussian random variables on $X$. ULA approximates $(\zeta_t)_{t\geq 0}$, but does not necessarily converge to $\mu$ in its longterm limit.

In practice, $N$ might be very large in which case we may not be able to repeatedly evaluate all $N$ gradients in \eqref{eq:ula}. Based on the popular Stochastic Gradient Descent method \citep{robbinsmonro} in optimisation , \citet{Teh2011} have proposed the \emph{Stochastic Gradient Langevin Dynamic}, which is of the form
\begin{equation} \label{eq:ula_dis}
    \tilde{\zeta}_{k+1} \leftarrow \tilde{\zeta}_{k} -  \eta \nabla\Phi_{i(k)}(\tilde{\zeta}_{k}) + \sqrt{2\eta} \xi_k,
\end{equation}
where $i(0), i(1),\ldots \sim \mathrm{Unif}(I)$ are iid.  This data subsampling that allows us to consider only one gradient at a time introduces again an additional error. In this work, we aim to study this subsampling error, isolatedly from the ULA error. This allows us to obtain a best case error for Langevin-based MCMC methods that are subject to subsampling and is independent from the discretisation.
To do so, we will consider the aforementioned \emph{Stochastic Gradient Langevin Diffusion}, a switched diffusion process that is given through the following dynamical system
\begin{align}\label{eq: switch}
    \mathrm{d}\theta_t = - \nabla \Phi_{\ii(t/\eta)}(\theta_t) \mathrm{d}t + \sqrt{2} \mathrm{d}W_t,
\end{align}
where $(\ii(t))_{t \geq 0}$ is a homogeneous continuous-time Markov process on $I$ that jumps from any state to any other state at rate $1$. Here, $\eta>0$ still has the character of a learning rate. The definition of the SGLDiff is especially 
motivated by earlier works \citep{Hanu,Jin2,Jin1,Latz} on continuous-time stochastic gradient descent  and different from purely diffusion-based analyses, e.g., such similar to \cite{LiTaiE}. We give examples for sample paths of SGLDiff with different learning rates $\eta$ in Figure~\ref{fig:SGLDiff}. There, we especially illustrate that SGLDiff $(\theta_t)_{t \geq 0}$ approximates the Langevin diffusion $(\zeta_t)_{t \geq 0}$, if $\eta \downarrow 0$. Moreover, we can see that SGLDiff also approximates our distribution of interest $\mu$. Throughout this work, we study this approximation of $(\zeta_t)_{t \geq 0}$ using $(\theta_t)_{t \geq 0}$.

\subsection{Contributions and outline}
We now state the contributions of this work and then give an outline.

From a \textbf{learning perspective}, we study the approximation of the  Langevin diffusion $(\zeta_t)_{t\geq 0}$ using SGLDiff $(\theta_t)_{t \geq 0}$.  Indeed, 
\begin{itemize}
\item we study convergence and divergence between Langevin dynamic $(\zeta_t)_{t\geq 0}$ and SGLDiff $(\theta_t)_{t\geq 0}$ for small $\eta$ and large $t$, respectively,
    \item we give assumptions under which SGLDiff has a unique stationary distribution $\mu^\eta$ and is ergodic, and we prove an error bound between this distribution $\mu^\eta$ and the target distribution $\mu$, and 
    \item we use the triangle inequality to then also bound the distance between SGLDiff $(\theta_t)_{t\geq 0}$ and target distribution $\mu$, giving us information about bias and convergence at the same time.
\end{itemize}
From a \textbf{probabilistic perspective}, we leverage the key ideas embedded within the ergodic theorem and show the strong convergence between SGLDiff $(\theta_t)_{t \geq 0}$ and Langevin dynamic $(\zeta_t)_{t \geq 0}$ while only having weak convergence between their coefficients $\nabla \Phi_{i(t/\eta)}$ and $\nabla \bar\Phi$. We adapt the reflection coupling method in the context of  switching diffusion processes and propose an innovative application of this method to address the convergence between the invariant measures $\mu$ and $\mu^\eta$ of systems (\ref{eq: limit}) and (\ref{eq: switch}), respectively.

We formulate the main results of this work in Theorems~\ref{Th: strong con}--\ref{Th: con inv} in Subsection~\ref{subs_main_resu} and bring them into context with discrete-in-time results in Subsection~\ref{subs_other}. We outline the proofs of Theorem~\ref{Th: strong con} and Theorems~\ref{prop: refl}--\ref{Th: con inv} in Sections~\ref{sec_strongconv} and \ref{sec_stationa} and make them rigorous in Appendices~\ref{app: strong conv} and \ref{app: con inv}, respectively. We conclude the work in 
Section~\ref{sec_conc} and point the reader towards related open problems.

\subsection{Main results} \label{subs_main_resu}
We present our main results in this section -- starting with two assumptions. 
\begin{assumptions}[Smoothness]\label{ass:  lipphi}
 For any $i\in I$, $\Phi_i \in \mathcal{C}^1(X:\R)$, i.e., it is continuously differentiable. In addition, $\nabla \Phi_i$ is Lipschitz continuous with Lipschitz constant $L$, i.e., for any $x, y\in X$,
$$\norm{\nabla \Phi_i(x) - \nabla \Phi_i(y)}\leq L \norm{x-y}.$$
\end{assumptions}
\begin{assumptions}\label{ass:  conv}
There exist $K,R>0$ such that for any $i\in I$ and $\norm{x-y}\ge R,$  $$\ip{\nabla \Phi_i(x)-\nabla \Phi_i(y),x-y}\ge K\norm{x-y}^2.$$
    \end{assumptions}
Assumption \ref{ass: lipphi} is usual in the literature \citep{pmlr-v65-raginsky17a, NEURIPS2018_9c19a2aa, pmlr-v161-zou21a} and provides existence and uniqueness of the solution to the equation (\ref{eq: switch}), see, e.g., \cite{XI2008588} or \citet[Chapter 2]{Yin2010}. Note that the solution is $\mathcal{F}^\eta_t:= 
\sigma(\mathcal{F}^B_t\cup \mathcal{F}^{I}_{t/\eta} )$-adapted, where $\mathcal{F}^B_t$ is the filtration generated by the Brownian motion $(W_t)_{t\geq0}$ and $\mathcal{F}^{I}_t$ is the filtration generated the Markov jump process $(\ii(t))_{t\geq0}$.
Assumption \ref{ass:  conv} is motivated by \cite{EBERLE20111101} and \cite{Eberle2016}, which allows us to use the reflection coupling method to prove exponential convergence. Intuitively, it states that the $\Phi_i$ appear strongly convex if $x$ and $y$ are far from each other for $i \in I$. We remark that while this assumption is weaker than strong convexity, it is stronger than the dissipativeness assumption, which is usually assumed in the discrete-in-time literature for convergence analysis of the non-convex case (e.g. \citealp{pmlr-v65-raginsky17a, NEURIPS2018_9c19a2aa, pmlr-v161-zou21a}). See Appendix \ref{app: assp} where we discuss this connection.

With the above assumptions, we show three convergence results regarding SGLDiff. We begin by showing that the processes $(\zeta_t)_{t \geq 0}$ and $(\theta_t)_{t \geq 0}$ may diverge as $t \rightarrow \infty$, but strongly converge at any fixed time $t$ if the learning rate $\eta \downarrow 0$.

\begin{theorem}\label{Th: strong con}
Let $(\theta_t)_{t \geq 0}$ be the solution to (\ref{eq: switch}) and $(\zeta_t)_{t \geq 0}$ be the solution to (\ref{eq: limit}) with initial value $\theta_0=\zeta_0$. Under Assumptions \ref{ass: lipphi}, we have the following inequality
    \begin{align*}
        \E[\norm{\theta_t-\zeta_t}]\le  C_{\Phi,\theta_0,d} \mathrm{e}^{8(1+L)t}\eta^{\frac{1}{4}},
    \end{align*}
  where  $C_{\Phi,\theta_0,d}= 8(1+d+\norm{\theta_0}^2+2\norm{\nabla \bar\Phi(0)}^2)^{\frac{1}{2}}C^{(1)}_{\Phi}$ and $C^{(1)}_{\Phi}=1+L+\sup_{i\in I}\norm{\nabla \Phi_i(0)}.$
\end{theorem}
Here, $C_{\Phi,\theta_0,d}$ depends linearly on $\sqrt{d}$. Next, we study the ergodicity of the SGLDiff \eqref{eq: switch}, which we will state in terms of the Wasserstein distance. The Wasserstein distance between two probability measures $\nu$ and $\nu'$  on $(X, \mathcal{B}(X))$ is given by
\begin{align*}
    \cW_{\norm{\cdot}}(\nu,\nu') =  \inf_{\Gamma\in\mathcal{H}(\nu,\nu') }\int_{X\times X}\norm{y-y'}\Gamma(\mathrm{d}y,\mathrm{d}y'),
\end{align*}
where $\mathcal{H}(\nu,\nu')$ is the set of coupling between $\nu$ and $\nu'$, i.e.
\begin{align*}
\mathcal{H}(\nu,\nu') = \{ 
    \Gamma \in \mathrm{Pr} (X \times X) : 
    \Gamma(A \times X) = \nu(A),  \\
    \Gamma(X \times B) = \nu'(B) \ 
    (A, B \in \mathcal{B}(X))
\}.
\end{align*}
We note that Assumptions \ref{ass: lipphi} and \ref{ass:  conv} imply that the joint process $(\theta_t, \ii(t))_{t \geq 0}$ defined in equation (\ref{eq: switch}) is Markovian and admits a unique invariant measure $M^\eta (\mathrm{d}\theta,\{i\})$, {see for example \cite{Yin2010}.} We denote by $\mu^\eta(\mathrm{d}\theta) := M^\eta (\mathrm{d}\theta,I)$ the $(\theta_t)_{t\geq0}$-marginal of the stationary distribution $M^\eta$ and, similarly, the  distributions  $\nu_t^\eta := \mathbb{P}(\theta_t \in \cdot)$ and $\nu_t := \mathbb{P}(\zeta_t \in \cdot)$ at a fixed time $t > 0$.
Finally, we assume in the following that $\ii(0) \sim \mathrm{Unif}(I)$ 
and then obtain the following ergodic theorem.
\begin{theorem}\label{prop: refl}
    Under the Assumptions \ref{ass: lipphi} and \ref{ass:  conv}, we have 
\begin{align*}
      \mathcal{W}_{\norm{\cdot}}(\nu^\eta_t,\mu^\eta)\le  C \mathrm{e}^{-c t}\mathcal{W}_{\norm{\cdot}}(\nu_0,\mu^\eta),
\end{align*}
where $c=\min\{3L+\frac{2}{R^2},K\}\mathrm{e}^{-{LR^2}/{2}}$ and $C=2 \mathrm{e}^{{LR^2}/{2}}$.
\end{theorem}
Theorem \ref{prop: refl} provides a quantitative way to measure the distance between the distribution of the state of SGLDiff $\theta_t$ at time $t > 0$ and its limiting measure, i.e. the exponential convergence between $\nu_t^\eta$ and $\mu^\eta$. Notice that the constants in the obtained upper bound are independent of the dimension as the reflection coupling reduces the diffusion to a one-dimensional Brownian motion, which will be explained later in the outline of the proof. 

In the third convergence result, we study the invariant measures $\mu$ and $\mu^\eta$ of $(\zeta_t)_{t \geq 0}$ and $(\theta_t)_{t \geq 0}$. Here, we bound the Wasserstein distance between  $\mu$ and $\mu^\eta$ and, thus, quantify the asymptotic subsampling error between correct distribution and SGLDiff.

\begin{theorem}\label{Th: con inv}
Under the Assumptions \ref{ass: lipphi} and \ref{ass:  conv}, the marginal distribution $\mu^\eta(dx)$ converges in $\mathcal{W}_{\norm{\cdot}}$ to the target distribution $\mu(dx)$, as $\eta \downarrow 0$. In particular, we have 
    \begin{align*}
        \mathcal{W}_{\norm{\cdot}}(\mu^\eta,\mu)\le  C_{\Phi,d}\eta^{c_\Phi},
    \end{align*}
    where $c_\Phi:= \frac{c}{32(L+1)+4c}$ and $C_{\Phi,d}:=C_{\Phi,\theta_0 = 0,d}+C^{(1)}_dC,$ with $C^{(1)}_d = \mathcal{O}(\sqrt{d}).$
\end{theorem}
When $t$ goes to infinity, both, $(\zeta_t)_{t \geq 0}$ and $(\theta_t)_{t \geq 0}$ converge to their invariant measures respectively, and this theorem shows that their invariant measures coincide as the learning rate goes to zero. From Theorem~\ref{Th: strong con}, we know that the dimension-dependence of the constant $C_{\Phi,d}$ is of order $\mathcal{O}(\sqrt{d})$. Moreover, we have the dimension-independent rate $c_\Phi = 1/4 - \delta$ for some $\delta>0$. The constant $C^{(1)}_d$ is discussed explicitly in Lemma \ref{lem: mu bound}.

\subsection{Comparison with discrete-in-time Langevin algorithm and related work} \label{subs_other}
 
There has been an increasing interest in the use of Langevin diffusion-based algorithms for the approximation of Bayesian posterior distributions as these algorithms have demonstrated significant potential for achieving accurate and efficient sampling \citep{Teh2011}. The convergence rate has been studied extensively under different log-concavity conditions on the target distribution, see for example \cite{pmlr-v65-dalalyan17a, Dalalyan2017, Durmus2016SamplingFA, Durmus2017, pmlr-v99-mangoubi19a}; as well as in the non-log-concave case, see for example, \cite{pmlr-v178-balasubramanian22a,   NEURIPS2018_c6ede20e, Chi2019, pmlr-v65-raginsky17a,NEURIPS2019_65a99bb7, NEURIPS2018_9c19a2aa, pmlr-v134-lamperski21a, Majka2018NonasymptoticBF, zhang2022nonasymptotic}. In recent years, there has been a growing body of research focused on improving and extending Langevin diffusion-based algorithms for Bayesian sampling. The subsampling-variant of the unadjusted Langevin algorithm, referred to as Stochastic Gradient Langevin Dynamics (SGLD), has proven to be particularly useful for sampling and optimization tasks in which the objective function is nonconvex, noisy, and/or has a large number of parameters. Recall that the Stochastic Gradient Langevin Dynamics updates are defined as in \eqref{eq:ula_dis}.
The convergence rate of this algorithm and its variants have been studied in for example, \cite{chen2019, pmlr-v119-deng20b,gao2022,pmlr-v65-raginsky17a, NEURIPS2018_9c19a2aa, Zhang2017AHT,    pmlr-v161-zou21a}. Since then, a significant amount of effort has been put into improving various aspects. For example, SGLD can be combined with variance reduction resulting in a faster convergence rate, such as the Stochastic Variance Reduced Gradient Langevin Dynamics (SVRG-LD), see for instance, \cite{NIPS2016_9b698eb3, HuangBecker2021, kinoshita2022improved,  NEURIPS2018_9c19a2aa,  XUzou2018, pmlr-v89-zou19a, pmlr-v161-zou21a}. Another direction of work are higher order MCMC methods, such as Hamiltonian Monte Carlo (see e.g. non-subsampling: \citealp{Bou-Rabee2020, ChenRaaz2020, durmus2019convergence, Mangoubi2018DimensionallyTB, mangoubi2019nonconvex, Neal}, subsampling: \citealp{pmlr-v139-zou21b}) and the underdamped Langevin dynamics (see e.g. non-subsampling: \citealp{ChengChatterji, Eberle2016, EberleGuillin2017}, subsampling: \citealp{ChenDing2015, ChenWangZhang2017, gao2022, pmlr-v80-zou18a, NEURIPS2019_c3535feb, akyildiz2024nonasymptotic}). {For dependent data, see for example \cite{chau2021stochastic}.}

{The convergence rate of the vanilla SGLD in the context of non-convex learning has been studied in several works; see for example, \cite{pmlr-v65-raginsky17a, pmlr-v161-zou21a, li2023geometric, Majka2018NonasymptoticBF, zhang2022nonasymptotic}. 
In particular, under similar conditions as our Assumption \ref{ass:  lipphi} and Assumption \ref{ass:  conv}, \cite{Majka2018NonasymptoticBF} obtained 
\begin{align*}
\mathcal{W}_{\norm{\cdot}}&(\mathbb{P}(\tilde{\zeta}_k\in \cdot),\mu) \\ &\leq 
e^{R^2} (1-m\mathrm{e}^{-R^2}\eta)^k + \mathcal{O}\left(e^{R^2}\sqrt{d}\right)\sqrt{\eta},
\end{align*}
where  $R$ is the constant from contractivity-at-infinity condition which is same as the $R$ defined in our Assumption \ref{ass:  conv} and $m>0$ is independent of $d$ and $R$. Although $R$ does not depend on the dimension $d$, there could be a relation between $R$ and $d$ in practice, which makes these constants dimension-dependent, see \cite{Bou-Rabee2020} for details.} Even though a direct comparison between our result and this bound may not be possible, it is still noteworthy to observe the analogous error in our continuous scenario, which offers interesting insights. Using the triangle inequality to combine Theorems \ref{prop: refl} and \ref{Th: con inv}, we have
$$\mathcal{W}_{\norm{\cdot}}(\nu_{t}^\eta,\mu)\le  \mathcal{O}(e^{R^2}\sqrt{d})\eta^{c_\Phi} + \mathcal{O}(\mathrm{e}^{R^2})\mathrm{e}^{-\mathcal{O}(\mathrm{e}^{-R^2}) t}.$$   The rate in the first term is independent of time, however $c_\Phi \leq 1/4$ indicates slow convergence. The second term decays exponentially in time and it is independent of the dimension $d$, {but depends on $R$, which, as discussed above, may depend on the dimension $d$.} 
{Hence, our analysis indicates that the ergodic rate does hardly suffer from the Euler-Maruyama discretisation that turns SGLDiff into SGLD; neither does the immediate dimension-dependence of the bias. We do see a similar dependence of $R$ on the bias term in SGLDiff, but also a slower convergence in terms of the learning rate parameter $\eta$. Thus, SGLD already achieves an error close to what we obtain in our best-case error analysis obtain. }
}

\section{APPROXIMATING THE LANGEVIN DIFFUSION
} \label{sec_strongconv}
In this section, we give a sketch of the proof of Theorem~\ref{Th: strong con} showing the strong convergence of $\theta_t \rightarrow \zeta_t$ for a fixed time $t > 0$, as $\eta \downarrow 0$. The full proof of this theorem and proofs of auxiliary results stated here are deferred to Appendix \ref{app: strong conv}. 
The  proof of Theorem \ref{Th: strong con} is inspired by the calculation of the variance for ergodic averages, for example, see \citet[Chapter 2.2]{Eberle2023} and \cite{kushner1984}. We notice that $\ii(\cdot/\eta)$ converges weakly to its invariant measure when $\eta\downarrow 0$. From the ergodic theory for Markov processes, however, we  expect that $\int_0^t \nabla \Phi_{\ii(s/\eta)}(\theta_s)ds= \eta\int_0^{t/\eta} \nabla \Phi_{\ii(r)}(\theta_{\eta r})dr$ converges to $\int_0^t \nabla \bar\Phi(\theta_s)ds$ strongly, which we can then use to prove strong convergence of the full processes. Before sketching the proof of Theorem \ref{Th: strong con}, we require some auxiliary results. We start with the following.
\begin{lemma}\label{lem: th3 2}
Under Assumption \ref{ass: lipphi},  for any $t>0,$ we have the following inequality,
    \begin{align*}
        \E[\norm{\zeta_t}^2]\le \tilde c_{t,\theta_0,d},
    \end{align*}
    where $\tilde c_{t,\theta_0,d}=\Big(\norm{\theta_0}^2+2\norm{\nabla \bar\Phi(0)}^2+2td\Big)\mathrm{e}^{2(L+1) t}.$
\end{lemma}

Lemma \ref{lem: th3 2} provides the boundedness of $(\zeta_t)_{t \geq 0}$ which will be used repeatedly in the rest of the paper.
The following Lemma shows that $(\zeta_t)_{t \geq 0}$ is continuous in time due to the continuity from the drift and the Brownian motion. This continuity allows us to employ a time decomposition later in the proof of Theorem \ref{Th: strong con}. 

\begin{lemma}\label{lem: th3 1}
     Under Assumption \ref{ass: lipphi},   $(\zeta_t)_{t\ge 0}$ is continuous w.r.t time, in the following sense: for  $t > s> 0,$ we have
     \begin{align*}
         \E[\norm{\zeta_t-\zeta_s}^2]\le c_{t,\theta_0,d}\abs{t-s},
     \end{align*}
where $c_{t,\theta_0,d}:= 2\mathrm{e}^{2(L+1) t}\tilde c_{t,\theta_0,d}.$
\end{lemma}

Notice that the Markov process  $(\ii(t))_{t \geq 0}$ is ergodic, i.e. 
$$\frac{1}{T}\int_0^T g_{\ii(t)} dt \to \frac{1}{N}\sum^N_{i=1}g_i,$$
as $T \rightarrow \infty$,
for some function $g: I \rightarrow X$.
The following lemma discusses the precise convergence rate and shows that the time average converges to the space-average with order  $\mathcal{O}(1/\sqrt{T})$.
\begin{lemma}\label{lem: th3 3}
Let $g: I \rightarrow X$ satisfy  $\sum_{i=1}^Ng_i=0$. Then
    \begin{align*}
        \sup_{\ii(0)\in I}\E_{\ii(0)}\Big[\norm{\int_0^{\frac{t}{\eta}}g_{\ii(s)}\mathrm{d}s}^2\Big]\le \frac{2\max_{i=1,...,N}\norm{g_i}^2}{N} \frac{t}{\eta}\ .
    \end{align*}
\end{lemma}

We now have all ingredients to explain how  Theorem \ref{Th: strong con} can be proven.
\subsection*{Proof sketch  of Theorem \ref{Th: strong con}} 
In the proof of Theorem \ref{Th: strong con}, the main idea is to break down the difference of $\theta_t$ and $\zeta_t$. First, we examine equations (\ref{eq: limit}) and (\ref{eq: switch}), we have 
\begin{align*}
    \norm{\theta_t-\zeta_t}=& \norm{\int_0^t \nabla \Phi_{\ii(s/\eta)}(\theta_s)-\nabla \bar\Phi(\zeta_s)\mathrm{d}s}\nonumber\\
    \le&  \norm{\int_0^t \nabla \Phi_{\ii(s/\eta)}(\theta_s)-\nabla \Phi_{\ii(s/\eta)}(\zeta_s)\mathrm{d}s}\\
    &+\norm{\int_0^t \nabla \Phi_{\ii(s/\eta)}(\zeta_s)-\nabla \bar\Phi(\zeta_s)\mathrm{d}s}.
\end{align*}

For the term $\nabla \Phi_{\ii(s/\eta)}(\theta_s)-\nabla \Phi_{\ii(s/\eta)}(\zeta_s)$, we apply the Lipschitz assumption from Assumption \ref{ass: lipphi}. Consequently, $\norm{\theta_t-\zeta_t}$ can be bounded by the sum of $\int_0^t\norm{\theta_s-\zeta_s} \mathrm{d}s$ and $\norm{\int_0^t \nabla \Phi_{\ii(s/\eta)}(\zeta_s)-\nabla \bar\Phi(\zeta_s)\mathrm{d}s}$.
Our main goal is to show that the second term is bounded by a constant depending on $t\eta^{1/4}$. To achieve this, we use a discretization technique to estimate the integral (\citealp{kushner1984} and \citealp[proof of Theorem 3]{Jin1}). More precisely, one can understand the switching rate $\eta$ as the discretization time-step and analyze the difference on each time interval of length $\tilde{\eta}$,
\begin{align*}
   \int_0^t \nabla &\Phi_{\ii(s/\eta)}(\zeta_s)-\nabla \bar\Phi(\zeta_s)\mathrm{d}s= \\&\sum^{1/\tet}_{j=1}\int_{(j-1)t\tet}^{jt\tet}\nabla \Phi_{\ii(s/\eta)}(\zeta_s)-\nabla \bar\Phi(\zeta_s)\mathrm{d}s.
\end{align*}
Within each time interval $((j-1)t\tet,jt\tet]$ we want to control the variation of $(\zeta_t)_{t \geq 0}$ (using Lemma~\ref{lem: th3 1}), which requires the length $\tilde{\eta}$ to be small. On the other hand, using the ergodicity bound from Lemma~\ref{lem: th3 3}, the fluctuation on each interval has to be large enough so that the overall sum goes to zero.
Consequently, we choose $\tet:= 1/\lfloor1/\sqrt{\eta}\rfloor$, where $\lfloor x \rfloor$ is the greatest integer less than or equal to $x$. 
This $\tet \ (\approx\sqrt{\eta})$ optimally satisfies those requirements.
Once this bound is established, we apply Gr\"onwall's inequality to obtain the desired result. 

\section{APPROXIMATING THE DISTRIBUTION OF INTEREST
} \label{sec_stationa}
We now study how well $\mu^\eta$ approximates $\mu$. Again, the full proofs of the main theorems and lemmas are deferred to Appendix \ref{app: con inv}. We begin by showing that $(\theta_t)_{t \geq 0}$ converges exponentially to its stationary measure.

\subsection*{Proof sketch  of Theorem \ref{prop: refl}} 
Before discussing the proof of Theorem \ref{prop: refl}, we recall the exponential contractivity for Markov semi-groups (see e.g. \citealp{EBERLE20111101, Eberle2023,  Latz}). Let $p_t: X \times \mathcal{B}(X) \rightarrow [0,1]$ be a homogeneous Markov semi-group and let $\pi$ be its invariant measure. The exponential contraction in Wasserstein distance induced by some distance $d$ is defined as
\begin{align*}
    \mathcal{W}_d(\pi_0 p_t,\pi )\le   \mathrm{e}^{-c t}\mathcal{W}_d(\pi_0,\pi).
\end{align*}
Now, while the pair $(\theta_t, \boldsymbol{i}(t/\eta))_{t\ge 0}$ is a Markov process, $(\theta_t)_{t\ge 0}$ on its own is not Markovian. Rather than exploring the contractivity of the pair $(\theta_t, \boldsymbol{i}(t/\eta))_{t\ge 0}$, we start the dynamic with $\ii(0)$ being already distributed according to its invariant measure $\mathrm{Unif}(I)$ and study the contractivity only in $(\theta_t)_{t\ge 0}$.
When the potentials $(\Phi_i)_{i\in I}$ are strongly convex, this property is classical and we could use the method in e.g. \cite{Latz} to obtain it. More precisely, one can construct a coupled process starting from the invariant measure and run the same dynamic with the same diffusion process. 
However, in the non-convex case, we do not obtain enough decay solely from the potential hence we need to construct the coupling in a way such that the diffusion term offers extra decay.
 By selecting an appropriate distance function $F(\cdot)$, it is possible to achieve exponential contractivity even in non-convex potential cases. Here, we choose the distance function to be a supermartingale w.r.t $\mathcal{F}_t^\eta$ and equivalent to the Euclidean distance so that we get exponential decay under this distance and deduce the exponential decay in $\norm{\cdot}$. This idea is adapted from the reflection couplings discussed by \cite{EBERLE20111101, Eberle2016}. Intuitively, by diffusing the coupled process along the reflection, we compensate for the lack of decay in the drift.

The classical coupling method fails when the drift term is not strictly contractive. Consider, for instance, a one-dimensional Brownian motion on the torus.  In this case, the drift term is $0$ and the classical coupling fails: the difference between the two processes is always a constant. In reflection coupling, we {couple differently:  both processes are driven by the same Brownian motion; however,  until they meet, we point the driving processes in opposite directions. Then}, the difference between the processes is a one-dimensional Brownian motion and will hit $0$ eventually. In our case, the assumption ``strongly convex at infinity'' allows the processes not to move too far away from one another and to meet eventually.
 As a result, a {fast} 
 exponential decay rate can  be obtained in the $ \mathcal{W}_{\norm{\cdot}}$ distance. 

Now, we move on to show the error bound between the stationary distribution $\mu^\eta$ and the distribution of interest  $\mu$.

\subsection*{Proof sketch  of Theorem \ref{Th: con inv}} 
The following lemma shows that $\mu^{\eta}$ and $\mu$ are bounded in terms of their first absolute moments. 
Recall that $\mu^{\eta}$ is the marginal distribution of the invariant measure of $(\theta_t)_{t \geq 0}$ and $\mu$ is the invariant measure of $(\zeta_t)_{t \geq 0}$.
\begin{lemma}\label{lem: mu bound} Let $\delta_0$  be the Dirac measure concentrated at $0$.
    Under Assumptions \ref{ass: lipphi} and \ref{ass:  conv}, we have
    \begin{align*}
        \mathcal{W}_{\norm{\cdot}}(\delta_0,\mu^\eta)\le C^{(1)}_d,
    \end{align*}
    where $C^{(1)}_d=  \sqrt{C_\Phi K^{-1}d}$ and $C_\Phi= 2(L+K)R^2+ \sup_{i\in I}\frac{\norm{\nabla \Phi_i(0)}^2}{K}. $
    
    Specifically, when $N=1$ -- so there is no subsampling -- it is easy to conclude that 
    \begin{align*}
         \mathcal{W}_{\norm{\cdot}}(\delta_0,\mu)\le C^{(1)}_d.
    \end{align*}
\end{lemma}

We first insert $\nu^\eta_t$ and $\nu_t$ into the distance between $\mu^\eta$ and $\mu$. Using the triangle inequality, we find that 
\begin{align*}
      \mathcal{W}_{\norm{.}}(\mu^\eta,\mu)\le   &\mathcal{W}_{\norm{.}}(\mu^\eta,\nu^\eta_t)+ \mathcal{W}_{\norm{.}}(\nu^\eta_t,\nu_t) \\ &+ \mathcal{W}_{\norm{.}}(\nu_t,\mu).
\end{align*}
Essentially, the distance between the invariant measures  propagates through the distance between their dynamics, $\mathcal{W}_{\norm{\cdot}}(\nu^\eta_t,\nu_t)$. Assuming they have the same initial value, this can be controlled using Theorem \ref{Th: strong con} and we obtain an upper bound of order $\eta^{1/4}$.
Starting at $0$, from Theorem \ref{prop: refl}, we can bound $\mathcal{W}_{\norm{\cdot}}(\mu^\eta,\nu^\eta_t)$ and $\mathcal{W}_{\norm{\cdot}}(\nu_t,\mu)$ by $\mathcal{W}_{\norm{\cdot}}(\delta_0,\mu^\eta)+\mathcal{W}_{\norm{\cdot}}(\delta_0,\mu)$ with exponential decay, which are bounded due to Lemma \ref{lem: mu bound}. Hence the distance between the dynamic and its invariant measure is bounded in both (\ref{eq: limit}) and (\ref{eq: switch}).
Since the left-hand side is independent of $t$, we choose $t$ freely to obtain an optimal bound.  While the contractivity is obtained for each dynamic and their limiting measures, the distance between the dynamics accumulates as $t$ goes to infinity, and the precise rate is given in Theorem \ref{Th: strong con}. Hence, we design $t$ as a function of $\eta$ such that the overall bound goes to $0$ as $\eta$ goes to $0$.

\section{CONCLUSIONS AND OPEN PROBLEMS} \label{sec_conc}
Our analysis has shown that our idealised subsampling MCMC dynamic SGLDiff is able to approximate the distribution of interest $\mu$ at high accuracy. {We especially learnt that the convergence rate is dimension-independent, only the prefactors depend linearly on the square-root of the dimension of the sample space. However, given that the constant $R$ often depends on the dimension, it is difficult to obtain a fully dimension-independent statement, say, as for preconditioned Crank--Nicolson samplers \citep{Cotter}. Overall, we learn that the convergence behaviour of SGLDiff and SGLD are very similar, which might indicate that the Euler--Maruyama discretisation that SGLD is based on is appropriate for the given task. }

This result also shows the general usefulness of the SGLDiff as a continuous-time model for subsampling in Langevin-based algorithms and as an analytical tool in their analysis. We list some open questions that could be studied next in this framework below. 

\paragraph{Optimisation.}  SGLD can also be seen as a noisier version of the Stochastic Gradient Descent method \citep{robbinsmonro}, where additional Gaussian noise is added to the stochastic gradients to further regularize the optimisation problem. In this case, we would probably consider  equation (\ref{eq: switch}) with an inverse temperature $\beta > 0$, i.e. 
\begin{align}\label{eq: beta}
    \mathrm{d}\theta_t = - \nabla \Phi_{\ii(t/\eta)}(\theta_t) \mathrm{d}t + \sqrt{2\beta^{-1}} \mathrm{d}W_t.
\end{align}
The non-subsampled version of this equation (i.e. setting $\Phi_{\ii(t/\eta)} = \hPhi$) has invariant distribution  $\mu_\beta(\mathrm{d}\theta) \propto \mathrm{e}^{-\Phi(\theta)/\beta}\mathrm{d}\theta.$ With certain assumptions on the potential function $\hPhi$, $\mu^\beta$ converges to $\delta_{\theta_*}$ weakly as $\beta\to \infty$, where $\delta_{\theta_*}$ is the Dirac delta function concentrated in the global minimizer $\theta_*$ of $\hPhi$. {Indeed, by rescaling $\beta$ over time, we can use the associated SDE for global optimisation, see, e.g., \cite{miclo}.} {Going back to subsampling, w}e may now study the invariant distribution $\mu^\eta_\beta$ of the process $(\theta_t)_{t \geq 0}$ that solves \eqref{eq: beta}. Here, we especially ask, whether $\mu^\eta_\beta \rightarrow \delta_{\theta_*}$, if $\beta \uparrow \infty$ and $\eta \downarrow 0$. And thus, whether and how fast this noisier version of Stochastic Gradient Descent can find the global optimiser of $\hPhi$.  

\paragraph{Momentum. }Higher-order dynamics have shown to be very successful at optimisation, e.g. ADAM \citep{kingma2017adam}, and sampling, e.g. the previously mentioned Hamiltonian Monte Carlo. In our work, we can obtain a higher-order dynamic by including a momentum term in equation (\ref{eq: switch}) and, thus, obtain an \emph{underdamped Stochastic Gradient Langevin Diffusion}
\begin{align*}\label{eq: momentum}
\mathrm{d}X_t=&V_t\mathrm{d}t\\
    \mathrm{d}V_t =& -\gamma V_t\mathrm{d}t- \nabla \Phi_{\ii(t/\eta)}(X_t) \mathrm{d}t + \sqrt{2} \mathrm{d}W_t,
\end{align*}
for which we would study the convergence of the solution $(X_t)_{t \geq 0}$ analogous to that of $(\theta_t)_{t \geq 0}$. The momentum may help to explore complicated energy landscapes in Bayesian deep learning and may reduce the influence of the subsampling. Ideas from \cite{Jin2} might help the analysis.

\paragraph{Epochs and subsampling without replacement.} 
{In practice, the Stochastic Gradient Descent method and the Stochastic Gradient Langevin Dynamics is often employed with full passes through the data set, so-called \emph{epochs}.  Here, the index process $\ii(\cdot)$ actually picks only from subsampled data sets $i \in I$ that were not picked so far and is reset after passing through all of the data. The resulting index process $\ii(\cdot)$ samples without replacement until the end of an epoch, where it is reset. Then, $\ii(\cdot)$ is not Markovian per se, but could be lifted into the space $I^{|I|+1}$, where it would be able to track its past in the current epoch and, thus, be Markovian. We would be interested in seeing whether sampling in epochs can improve the convergence of SGLD(iff) and  Stochastic Gradient Descent.}

\subsubsection*{Acknowledgements}
The authors thank \"O.\ Deniz Akyildiz and Mateusz Majka for helpful discussions. The third author would like to thank the Isaac Newton Institute for Mathematical Sciences for support and hospitality during the programme \textit{The mathematical and statistical foundation of future data-driven engineering} when work on this paper was undertaken. This work was supported by: EPSRC Grant Number EP/R014604/1.

\nocite{}
\bibliographystyle{apalike}
\bibliography{references}

\clearpage

\begin{appendix}

\onecolumn
\aistatstitle{Subsampling Error in Stochastic Gradient Langevin Diffusions: \\
[0.2cm]Supplementary Materials}
In this supplementary material, we organize our content into three appendices. We provide the proofs for all Theorems and Lemmas from the main paper. Appendix \ref{app: strong conv} contains the proof of Theorem \ref{Th: strong con} from the main text. We prove Theorem \ref{Th: strong con} by first proving Lemma \ref{lem: th3 2}, Lemma \ref{lem: th3 1}, and Lemma \ref{lem: th3 3}.
Appendix \ref{app: con inv} has the proof of Theorem  \ref{prop: refl} and Theorem \ref{Th: con inv}.
Appendix \ref{app: assp} discusses the difference between the dissipativeness assumption and Assumption \ref{ass:  conv}.

\appendix
\section{Proof of Theorem \ref{Th: strong con}}\label{app: strong conv}
We prove Theorem \ref{Th: strong con} starting with showing Lemma \ref{lem: th3 2}.
\subsection{Proof of Lemma \ref{lem: th3 2}}
 The following lemma shows the boundedness of $(\zeta_t)_{t \geq 0}$ which will be used repeatedly in the rest of the paper.\\
 
\textbf{Lemma~\ref{lem: th3 2}} 
\emph{
Under Assumption \ref{ass: lipphi},  for any $t>0,$ we have the following inequality,
    \begin{align*}
        \E[\norm{\zeta_t}^2]\le \tilde c_{t,\theta_0,d},
    \end{align*}
    where $\tilde c_{t,\theta_0,d}=\Big(\norm{\theta_0}^2+2t\norm{\nabla \bar\Phi(0)}^2+2td\Big)e^{2(L+1) t}.$
    }
\begin{proof}
      By It\^o's formula, we have 
    \begin{align*}
        \frac{\norm{\zeta_t}^2}{2}=&  \norm{\theta_0}^2-\int_0^t\ip{\zeta_s,\nabla \bar\Phi(\zeta_s)} \mathrm{d}t+\sqrt{2}\int_0^t\ip{\zeta_s,\mathrm{d}B_s}+td\\
        =& \norm{\theta_0}^2-\int_0^t\ip{\zeta_s,\nabla \bar\Phi(\zeta_s)-\nabla \bar\Phi(0)}\mathrm{d}s \\
        &-\int_0^t  \ip{\zeta_s,\nabla \bar\Phi(0)}\mathrm{d}s+\sqrt{2}\int_0^t\ip{\zeta_s,\mathrm{d}B_s}+td\\
        \le& \norm{\theta_0}^2+ L\int_0^t\norm{\zeta_s}^2\mathrm{d}s+\norm{\nabla \bar\Phi(0)}\int_0^t  \norm{\zeta_s}\mathrm{d}s+\sqrt{2}\int_0^t\ip{\zeta_s,\mathrm{d}B_s}+td\\
        \le& \norm{\theta_0}^2+ (L+1)\int_0^t\norm{\zeta_s}^2\mathrm{d}s+t\norm{\nabla \bar\Phi(0)}^2+\sqrt{2}\int_0^t\ip{\zeta_s,\mathrm{d}B_s}+td.
    \end{align*}
    Taking expectation of both sides, we have
    \begin{align*}
         \frac{\E[\norm{\zeta_t}^2]}{2}\le \frac{\norm{\theta_0}^2}{2}+ (L+1)\int_0^t\E[\norm{\zeta_s}^2]\mathrm{d}s+t\norm{\nabla \bar\Phi(0)}^2+td.
    \end{align*}
    By using Gr\"onwall's inequality, we obtain the bound 
    \begin{align*}
        \E[\norm{\zeta_t}^2]\le \Big(\norm{\theta_0}^2+2t\norm{\nabla \bar\Phi(0)}^2+td\Big)e^{2(L+1) t},
    \end{align*}
   which completes the proof.
\end{proof}

\subsection{Proof of Lemma \ref{lem: th3 1}}

\textbf{Lemma~\ref{lem: th3 1}} 
\emph{               Under Assumption \ref{ass: lipphi},   $(\theta_t)_{t\ge 0}$ is continuous w.r.t time, in the following sense: for  $t > s> 0,$ we have
     \begin{align*}
         \E[\norm{\zeta_t-\zeta_s}^2]\le c_{t,\theta_0,d}\abs{t-s},
     \end{align*}
where $c_{t,\theta_0,d}:= 2\mathrm{e}^{2(L+1) t}\tilde c_{t,\theta_0,d}.$
     }
    \begin{proof}
From equation (\ref{eq: limit}), we get
    \begin{align*}
    \norm{\zeta_t-\zeta_s}\le  \underbrace{\int_s^t\norm{\nabla \bar\Phi(\zeta_r)} \mathrm{d}r}_{(m2.1)}+\underbrace{\sqrt{2}\norm{B_t-B_s}}_{(m2.2)}.
\end{align*}
 The second term can be bounded by the variance of increments of Brownian motions,
\begin{align*}
\E\Big[(m2.2)^2\Big]= 2\abs{t-s}.     
\end{align*}
Consider the first term,
\begin{align*}
  (m2.1) = \int_s^t\norm{\nabla \bar\Phi(\zeta_r)} \mathrm{d}r=& \int_s^t(\norm{\nabla \bar\Phi(\zeta_r)-\nabla\bar\Phi(0)}+\norm{\nabla\bar\Phi(0)}) \mathrm{d}r\\
    \le& L \int_s^t\norm{\zeta_r} \mathrm{d}r +  \norm{\nabla\bar\Phi(0)}\abs{t-s}.
\end{align*}
By Lemma \ref{lem: th3 2}, we conclude $$\E[\abs{(m2.1)}^2]\le 2(\tilde c_{t,\theta_0,d}+ \norm{\nabla\bar\Phi(0)}^2)t\abs{t-s},$$ 
which yields
\begin{align*}
    \E[\norm{\zeta_t-\zeta_s}^2]\le 2\E[\abs{(m2.1)}^2]+ 2\E[\abs{(m2.2)}^2]\le c_{t,\theta_0,d}\abs{t-s}
\end{align*}
for some constant $c_{t,\theta_0,d}.$
\end{proof}

\subsection{Proof of Lemma \ref{lem: th3 3}}
\textbf{Lemma~\ref{lem: th3 3}} 
\emph{Let $g: I \rightarrow X$ satisfy  $\sum_{i=1}^Ng_i=0$. Then
    \begin{align*}
        \sup_{\ii(0)\in I}\E_{\ii(0)}\Big[\norm{\int_0^{\frac{t}{\eta}}g_{\ii(s)}\mathrm{d}s}^2\Big]\le \frac{2\max_{i=1,...,N}\norm{g_i}^2}{N} \frac{t}{\eta}\ .
    \end{align*}
    }

\begin{proof}
    We rewrite the square integral and use the Markov property of $(\ii(t))_{t \geq 0}$,
    \begin{align*}
        \E_{\ii(0)}\Big[\norm{\int_0^{\frac{t}{\eta}}g_{\ii(s)}\mathrm{d}s}^2\Big]=& \E_{\ii(0)}\Big[\int_0^{\frac{t}{\eta}}\int_0^{\frac{t}{\eta}}\ip{g_{\ii(s)},g_{\ii(r)}}\mathrm{d}s\mathrm{d}r\Big]\\
        =& 2\E_{\ii(0)}\Big[\int_0^{\frac{t}{\eta}}\int_0^{\frac{t}{\eta}}\ip{g_{\ii(s)},g_{\ii(r)}}\boldsymbol{1}_{r\le s}\mathrm{d}s\mathrm{d}r\Big] \textit{ (since $s,r$ are symmetric)}\\
        =& 2\E_{\ii(0)}\Big[\int_0^{\frac{t}{\eta}}\int_r^{\frac{t}{\eta}}\ip{g_{\ii(s)},g_{\ii(r)}}\mathrm{d}s\mathrm{d}r\Big]\\
       =& 2\E_{\ii(0)}\Big[\int_0^{\frac{t}{\eta}}\int_r^{\frac{t}{\eta}}\E[\ip{g_{\ii(s)},g_{\ii(r)}}|\mathcal{F}_r]\mathrm{d}s\mathrm{d}r\Big] \\
        =& 2\E_{\ii(0)}\Big[\int_0^{\frac{t}{\eta}}\int_r^{\frac{t}{\eta}}\E_{j=\ii(r)}[\ip{g_{\ii(s-r)},g_{j}}]\mathrm{d}s\mathrm{d}r\Big]\textit{ (by Markov property)}\\
         =& 2\E_{\ii(0)}\Big[\int_0^{\frac{t}{\eta}}\int_r^{\frac{t}{\eta}}\frac{1-e^{-N (s-r)}}{N}\underbrace{\ip{\sum^N_{i=1}g_i,g_{\ii(r)}}}_{=0}
         +e^{-N (s-r)}\norm{g_{\ii(r)}}^2\mathrm{d}s\mathrm{d}r\Big]\\
         &\left(\frac{1-e^{-N (s-r)}}{N} \textit{is the probability switching from $j$ to any other state in $(s-r, s]$.}\right)\\
         =& 2\int_0^{\frac{t}{\eta}}\int_r^{\frac{t}{\eta}}e^{-N (s-r)}\E_{\ii(0)}[\norm{g_{\ii(r)}}^2] \mathrm{d}s \mathrm{d}r\\
         \le& 2\max_{i=1,...,N}\norm{g_i}^2 \int_0^{\frac{t}{\eta}}\int_0^{\frac{t}{\eta}-r}e^{-N m} \mathrm{d}m \mathrm{d}r\\
         \le& \frac{2\max_{i=1,...,N}\norm{g_i}^2}{N} \frac{t}{\eta}.
    \end{align*}
\end{proof}

\subsection{Proof of Theorem \ref{Th: strong con}}


\textbf{Theorem~\ref{Th: strong con}} 
\emph{
Let $(\theta_t)_{t \geq 0}$ be the solution to (\ref{eq: switch}) and $(\zeta_t)_{t \geq 0}$ be the solution to (\ref{eq: limit}) with initial value $\theta_0=\zeta_0$. Under Assumption \ref{ass: lipphi}, we have the following inequality
    \begin{align*}
        \E[\norm{\theta_t-\zeta_t}]\le  C_{\Phi,\theta_0,d} \mathrm{e}^{8(1+L)t}\eta^{\frac{1}{4}},
    \end{align*}
  where  
  $C_{\Phi,\theta_0,d}= 8(1+d+\norm{\theta_0}^2+2\norm{\nabla\bar\Phi(0)}^2)^{\frac{1}{2}}
C^{(1)}_{\Phi}$ and $C^{(1)}_{\Phi}=1+L+\sup_{i\in I}\norm{\nabla \Phi_i(0)}.$}

\begin{proof}
We decompose $\norm{\theta_t-\zeta_t}$ into two terms using equations (\ref{eq: limit}) and (\ref{eq: switch}), 
\begin{align}
    \norm{\theta_t-\zeta_t}=& \norm{\int_0^t \nabla \Phi_{\ii(s/\eta)}(\theta_s)-\nabla \bar\Phi(\zeta_s)\mathrm{d}s}\nonumber\\
    \le&  \norm{\int_0^t \nabla \Phi_{\ii(s/\eta)}(\theta_s)-\nabla \Phi_{\ii(s/\eta)}(\zeta_s)\mathrm{d}s}+\norm{\int_0^t \nabla \Phi_{\ii(s/\eta)}(\zeta_s)-\nabla \bar\Phi(\zeta_s)\mathrm{d}s}\nonumber\\
    \le& L\int_0^t\norm{\theta_s-\zeta_s} \mathrm{d}s+\norm{\int_0^t \nabla \Phi_{\ii(s/\eta)}(\zeta_s)-\nabla \bar\Phi(\zeta_s)\mathrm{d}s}.
\end{align}
We claim that $\E[\norm{\int_0^t \nabla \Phi_{\ii(s/\eta)}(\zeta_s)-\nabla \bar\Phi(\zeta_s)\mathrm{d}s}]$ can be bounded by $C_{t,\theta_0,d}\sqrt{\eta}$ for some $C_{t,\theta_0,d}>0$. 
Let $\tet:= 1/\lfloor1/\sqrt{\eta}\rfloor$, where $\lfloor x \rfloor$ is the greatest integer less than or equal to $x$.  Then we have the following decomposition
\begin{align*}
   \int_0^t \Big(\nabla \Phi_{\ii(s/\eta)}(\zeta_s)-\nabla \bar\Phi(\zeta_s)\Big)\mathrm{d}s= \sum^{1/\tet}_{i=1}\int_{(i-1)t\tet}^{it\tet}G(\ii(s/\eta), \zeta_s)\mathrm{d}s,
\end{align*}
where $G(i, x)= \nabla \Phi_i(x)-\nabla \bar\Phi(x).$ For fixed $i$, $G(i, x)$ is Lipschitz continuous with constant $L$. Hence,
\begin{align*}
   \norm{\int_{(i-1)t\tet}^{it\tet}(G(\ii(s/\eta), \zeta_s)\mathrm{d}s}\le& \norm{\int_{(i-1)t\tet}^{it\tet}(G(\ii(s/\eta), \zeta_s)-G(\ii(s/\eta), \zeta_{(i-1)t\tet}))\mathrm{d}s}\\
   +& \norm{\int_{(i-1)t\tet}^{it\tet}(G(\ii(s/\eta), \zeta_{(i-1)t\tet})\mathrm{d}s}\\
   \le& L \underbrace{\int_{(i-1)t\tet}^{it\tet}\norm{\zeta_s-\zeta_{(i-1)t\tet}}\mathrm{d}s}_{(p2.1)}+ \underbrace{\norm{\int_{(i-1)t\tet}^{it\tet}(G(\ii(s/\eta), \zeta_{(i-1)t\tet})\mathrm{d}s}}_{(p2.2)}.
\end{align*}
By Lemma \ref{lem: th3 1}, we bound the first term as
\begin{align*}
    \E[(p2.1)]\le L(c_{t,\theta_0,d})^{\frac{1}{2}}(t\tet)^{\frac{3}{2}}.
\end{align*}
We first study the second term whilst conditioning on $\mathcal{F}^\eta_{(i-1)t\tet}$, 
\begin{align*}
    \E\Big[(p2.2)\Big|\mathcal{F}^\eta_{(i-1)t\tet}\Big]&=\E_{i^\eta((i-1)t\tet),x=\zeta_{(i-1)t\tet}}\Big[\norm{\int_0^{t\tet}(G(\ii(s/\eta), x)\mathrm{d}s}\Big]\\
    &\le \Big[\E_{i^\eta((i-1)t\tet),x=\zeta_{(i-1)t\tet}}\norm{\int_0^{t\tet}(G(\ii(s/\eta), x)\mathrm{d}s}^2\Big]^{\frac{1}{2}}\\
    &\overset{r= s/\eta}{=} \Big[\E_{i^\eta((i-1)t\tet),x=\zeta_{(i-1)t\tet}}\eta^2\norm{\int_0^{t\tet \eta^{-1}}(G(\ii(r), x)\mathrm{d}r}^2\Big]^{\frac{1}{2}}\\
    &\underbrace{\le}_{\text{Lemma \ref{lem: th3 3}}} \frac{2\max_{j=1,...,N}\norm{G(j, \zeta_{(i-1)t\tet})}}{\sqrt{N}} \sqrt{t\eta \tet}\\
    &=\frac{2\max_{j=1,...,N}\norm{(\nabla\Phi_j-\nabla\bar\Phi)( \zeta_{(i-1)t\tet})}}{\sqrt{N}} \sqrt{t\eta \tet}\\
    &\le C^{(1)}_\Phi(1+\norm{\zeta_{(i-1)t\tet}})\sqrt{t\eta \tet},
\end{align*}
where $C^{(1)}_\Phi= 2(1+L+\sup_{i\in I}\norm{\nabla \Phi_i(0)}).$
By Lemma \ref{lem: th3 2}, this implies $$\E[(p2.2)]\le C^{(1)}_\Phi (c_{t,\theta_0,d})^{\frac{1}{2}}\tet^{\frac{3}{2}}. $$ 
Hence,
\begin{align*}
    \E[\norm{\theta_t-\zeta_t}]\le L\int_0^t\E[\norm{\theta_s-\zeta_s}] \mathrm{d}s+ C^{(1)}_\Phi (c_{t,\theta_0,d})^{\frac{1}{2}}(1+\sqrt{t}) \tet^{\frac{1}{2}}.
\end{align*}
Using Gr\"onwall's inequality yields 
\begin{align*}
    \E[\norm{\theta_t-\zeta_t}]\le C^{(1)}_\Phi (c_{t,\theta_0,d})^{\frac{1}{2}}(1+\sqrt{t})  \tet^{\frac{1}{2}}e^{Lt}.
\end{align*}
Recall that $(c_{t,\theta_0,d})^{\frac{1}{2}}(1+\sqrt{t})=2(1+\sqrt{t})(1+2td+\norm{\theta_0}^2+2\norm{\nabla \bar\Phi(0)}^2)^{\frac{1}{2}}e^{4(L+1) t}$. Therefore,
\begin{align*}
     \E[\norm{\theta_t-\zeta_t}]\le C_{\Phi,\theta_0,d} e^{8(L+1) t}\eta^{\frac{1}{4}},
\end{align*}
where $ C_{\Phi,\theta_0,d}= 8(1+d+\norm{\theta_0}^2+2\norm{\nabla \bar\Phi(0)}^2)^{\frac{1}{2}}(1+L+\sup_{i\in I}\norm{\nabla \Phi_i(0)})$.
\end{proof}
{We remark that the factor in $t$ can be improved with an additional assumption but the rate in $\eta$ remains unchanged. In fact, with Assumption \ref{ass:  conv}, the bound can be improved to $ (\mathrm{e}^{8(1+L)t}\eta^{1/4})\land (R+\eta^{1/4})$.}

\section{Proof of Theorem  \ref{prop: refl} and Theorem \ref{Th: con inv}}\label{app: con inv}

\subsection{Proof of Theorem \ref{prop: refl}}
\textbf{Theorem~\ref{prop: refl}} 
\emph{
 Under  Assumptions \ref{ass: lipphi} and \ref{ass:  conv}, we have 
\begin{align*}
      \mathcal{W}_{\norm{\cdot}}(\nu^\eta_t,\mu^\eta)\le  C \mathrm{e}^{-c t}\mathcal{W}_{\norm{\cdot}}(\nu_0,\mu^\eta),
\end{align*}
where $c=\min\{\big(3L+\frac{2}{R^2}\big),K\}\mathrm{e}^{-{LR^2}/{2}}$ and $C=2 \mathrm{e}^{{LR^2}/{2}}$.
}
\begin{proof}
    We adapt the reflection coupling method introduced in \cite{EBERLE20111101, Eberle2016}. Let $(\theta_t)_{t\ge 0}$ be the solution  to equation (\ref{eq: switch}) with $\theta_0\sim \nu.$  In the coupling approach, we construct another solution $(\tilde\theta_t)_{t\ge 0}$ of the same SDE on the same probability space with the same index process $(\ii(t/\eta))_{t\ge 0}$ and with a different initial law in $\theta$ denoted as $\tilde\theta_0\sim \mu^\eta$,
    i.e.
    \begin{equation}\label{eq:AS:coupling}
\left\{\begin{array}{rl}
d\theta_t &= -\nabla \Phi_{\ii(t/\eta)}(\theta_t) \mathrm{d}t+\sqrt{2}\mathrm{d} B_t\\
d\tilde\theta_t &= -\nabla \Phi_{\ii(t/\eta)}(\tilde\theta_t) \mathrm{d}t+\sqrt{2}\mathrm{d}\tilde B_t\\
i(t = 0) &= i_0\\ 
\tilde\theta(t = 0) &= \tilde\theta_0\sim  \mu^\eta,\ \theta(t = 0) = \theta_0\sim  \nu
\end{array}\right.
\end{equation}
where 
\begin{align*}
    \tilde B_t= \int_0^t (I_d-2e_se_s^T\boldsymbol{1}_{\theta_s\ne \tilde\theta_s})dB_s,\ \ e_s= (\theta_s-\tilde\theta_s)/\norm{\theta_s-\tilde\theta_s},
\end{align*}
and $I_d$ is the identity matrix of dimension $d$.
It is not hard to verify $I_d-2e_se_s^T$ is an orthogonal matrix, which implies that $\tilde B_t$ is a d-dimensional Brownian motion.

Let $T=\inf\{t\ge 0: \theta_t=\tilde\theta_t \}$ and $r_t=\norm{ \theta_t-\tilde\theta_t},$ then for $t<T,$ the difference between $\theta_t$ and $\tilde\theta_t$ satisfies 
\begin{align}\label{ito F} 
    d(\theta_t-\tilde\theta_t)= -(\nabla \Phi_{\ii(t/\eta)}(\theta_t)-\nabla \Phi_{\ii(t/\eta)}(\tilde\theta_t)) \mathrm{d}t+ 2\sqrt{2} e_t  \mathrm{d}B_t^1,
\end{align}
where $B_t^1:= \int_0^t e_s\cdot dB_s$, which is a one-dimensional Brownian motion. Hence, for $F\in C^2(\R),$ by It\^o's formula, we have, for $t<T,$ 
\begin{align*}
    dF(r_t) =\Big[-\ip{e_t, \nabla \Phi_{\ii(t/\eta)}(\theta_t)-\nabla \Phi_{\ii(t/\eta)}(\tilde\theta_t)}F'(r_t)+4 F''(r_t)\Big]dt+2\sqrt{2}F'(r_t) dB_t^1.
\end{align*}
We choose $F(r)= \int_0^r e^{-\frac{L\min\{s,R\}^2}{2}}(1-\frac{1}{2R}\min\{s,R\})ds.$ Note that $F'$ is non-increasing. Hence, $e^{-\frac{LR^2}{2}}r/2\le F(r)\le r.$ 
Next, we are going to verify that for some constant $c>0,$
\begin{align}\label{function F}
    (L\boldsymbol{1}_{r\le R}-K\boldsymbol{1}_{r> R})rF'(r)+4F''(r)\le -c F(r).
\end{align}
When $r>R$, since $F''(r)\le 0$ and $F'(r)= e^{-\frac{LR^2}{2}}$, (\ref{function F}) holds with constant $c\le K e^{-\frac{LR^2}{2}}.$ For $r\le R,$ we have $F'(r)=e^{-\frac{Lr^2}{2}}(1-\frac{r}{2R})$ and $F''(r)=e^{-\frac{Lr^2}{2}}(-\frac{2Lr}{2}+\frac{Lr^2}{2R}-\frac{1}{2R}).$ Hence, for $r\le R,$ the left side of (\ref{function F}) is 
\begin{align*}
     L\boldsymbol{1}_{r\le R}rF'(r)+4F''(r)=&e^{-\frac{Lr^2}{2}}r\big(L-\frac{Lr}{2R}-4L+\frac{Lr}{2R}-\frac{2}{rR}\big)\\
     \le& -e^{-\frac{Lr^2}{2}}r\big(3L+\frac{2}{rR}\big)\le -\big(3L+\frac{2}{R^2}\big)e^{-\frac{LR^2}{2}}F(r).
\end{align*}
Setting $c=\min\{\big(3L+\frac{2}{R^2}\big),K\}e^{-\frac{LR^2}{2}}$ yields  inequality (\ref{function F}). By Assumptions \ref{ass: lipphi} and \ref{ass:  conv}, since $r_t=0$ for $t\ge T,$ we know $e^{c t}F(r_t)$ is a supermartingale w.r.t $\mathcal{F}^\eta_t$. Therefore,
\begin{align*}
    \E[F(r_t)]\le e^{-c t}\E[F(r_0)].
\end{align*}
Recall that  $e^{-\frac{LR^2}{2}}r/2\le F(r)\le r,$ we get 
$$\mathcal{W}_{\norm{.}}(\nu^\eta_t,\tilde\nu^\eta_t)\le  C e^{-c t}\mathcal{W}_{\norm{.}}(\nu_0,\mu^\eta)$$ 
for $C=2 e^{\frac{LR^2}{2}}. $  Since $\mu^\eta$ is invariant in time, we have $\tilde \nu^\eta_t = \nu^\eta_0 = \mu^\eta(\cdot ,I)$,
which completes the proof.
\end{proof}

\subsection{Proof of Lemma \ref{lem: mu bound}}

\textbf{Lemma~\ref{lem: mu bound}} 
\emph{ Let $\delta_0$  be the Dirac delta function at $0$.
    Under Assumptions \ref{ass: lipphi} and \ref{ass:  conv}, we have
    \begin{align*}
        \mathcal{W}_{\norm{.}}(\delta_0,\mu^\eta)\le C^{(1)}_d,
    \end{align*}
    where $C^{(1)}_d=  \sqrt{C_\Phi K^{-1}d}$ and $C_\Phi= 2(L+K)R^2+ \sup_{i\in I}\frac{\norm{\nabla \Phi_i(0)}^2}{K}. $}

\emph{    
Specifically, when $N=1$, it is easy to conclude that 
\begin{align*}
     \mathcal{W}_{\norm{.}}(\delta_0,\mu)\le C^{(1)}_d.
\end{align*}
    }

\begin{proof}
    Let $\nu^\eta_t$ be the distribution of $\theta_t$ with $(\theta_0,I_0)=(0,i_0)$ and $i_0\sim \mathrm{Unif}(I)$, we have
    \begin{align*}
        \mathcal{W}_{\norm{.}}(\delta_0,\mu^\eta)\le \mathcal{W}_{\norm{.}}(\delta_0,\nu^\eta_t)+\mathcal{W}_{\norm{.}}(\nu^\eta_t,\mu^\eta).
    \end{align*}
      From Theorem \ref{prop: refl}, we can bound the second term via 
    \begin{align*}
    \mathcal{W}_{\norm{.}}(\nu^\eta_t,\mu^\eta)\le    C e^{-c t}\mathcal{W}_{\norm{.}}(\delta_0,\mu^\eta).
    \end{align*}
    For the first term, by It\^o's formula, we have
    \begin{align*}
        \mathrm{d}\norm{\theta_t}^2= &-2\ip{\theta_t,\nabla \Phi_{\ii(t/\eta)}(\theta_t)} \mathrm{d}t+ 2\sqrt{2}\ip{\theta_t,\mathrm{d}B_t}+2d\mathrm{d}t\\
        = &-2\ip{\theta_t,\nabla \Phi_{\ii(t/\eta)}(\theta_t)-\nabla \Phi_{\ii(t/\eta)}(0)} \mathrm{d}t \\ &  -2\ip{\theta_t,\nabla \Phi_{\ii(t/\eta)}(0)}dt+ 2\sqrt{2}\ip{\theta_t,\mathrm{d}B_t}+2d\mathrm{d}t.
    \end{align*}
    Moreover,
    \begin{align*}
       -2&\ip{\theta_t,\nabla \Phi_{\ii(t/\eta)}(\theta_t)-\nabla \Phi_{\ii(t/\eta)}(0)}-2\ip{\theta_t,\nabla \Phi_{\ii(t/\eta)}(0)}\\
       \le& 2L\norm{\theta_t}^2\boldsymbol{1}_{\norm{\theta_t}\le R}-2K\norm{\theta_t}^2\boldsymbol{1}_{\norm{\theta_t}> R}+K\norm{\theta_t}^2+ \frac{\norm{\nabla \Phi_{\ii(t/\eta)}(0)}^2}{K}\\
        \le& 2(L+K)\norm{\theta_t}^2\boldsymbol{1}_{\norm{\theta_t}\le R} -K\norm{\theta_t}^2+ \frac{\norm{\nabla \Phi_{\ii(t/\eta)}(0)}^2}{K}\\
        \le& 2(L+K)R^2 -K\norm{\theta_t}^2+ \frac{\norm{\nabla \Phi_{\ii(t/\eta)}(0)}^2}{K}\\
        \le& C_\Phi-K\norm{\theta_t}^2,
    \end{align*}
    where $C_\Phi= 2(L+K)R^2+ \sup_{i\in I}\frac{\norm{\nabla \Phi_i(0)}^2}{K}. $
    
    Since we set $\theta_0=0,$ we have
    \begin{align*}
        e^{Kt}\E[\norm{\theta_t}^2]\le C_\Phi d\int_0^t e^{Ks}ds,
    \end{align*}
    which implies $\mathcal{W}_{\norm{.}}(\delta_0,\nu^\eta_t)\le \sqrt{C_\Phi K^{-1}d}$. Therefore,
    \begin{align*}
        \mathcal{W}_{\norm{.}}(\delta_0,\mu^\eta)\le \sqrt{C_\Phi K^{-1}d}+C e^{-c t}\mathcal{W}_{\norm{.}}(\delta_0,\mu^\eta).
    \end{align*}
    The second term goes to $0$ as $t\to \infty$, which yields the proof.
\end{proof}

\subsection{Proof of Theorem \ref{Th: con inv}}


\textbf{Theorem~\ref{Th: con inv}} 
\emph{
Under the Assumptions \ref{ass: lipphi} and \ref{ass:  conv}, the marginal distribution $\mu^\eta(dx)$ converges weakly to the stationary measure of $(\zeta_t)_{t\geq0}$. In particular, we have 
    \begin{align*}
        \mathcal{W}_{\norm{\cdot}}(\mu^\eta,\mu)\le  C_{\Phi,d}\eta^{c_\Phi},
    \end{align*}
    where $c_\Phi:= \frac{c}{32(L+1)+4c}$ and $C_{\Phi,d}:=C_{\Phi,\theta_0 = 0,d}+C^{(1)}_dC,$ with $C^{(1)}_d = \mathcal{O}(\sqrt{d}).$
}
\begin{proof}
      We first bound $  \mathcal{W}_{\norm{.}}(\mu^\eta,\mu)$ by
    \begin{align*}
          \mathcal{W}_{\norm{.}}(\mu^\eta,\mu)\le   \underbrace{\mathcal{W}_{\norm{.}}(\mu^\eta,\nu^\eta_t)}_{(w1.1)}+ \underbrace{\mathcal{W}_{\norm{.}}(\nu^\eta_t,\nu_t)}_{(w1.2)}+ \underbrace{\mathcal{W}_{\norm{.}}(\nu_t,\mu)}_{(w1.3)},
    \end{align*}
    where $\nu^\eta_0=\nu_0=\delta_0.$ From Theorem  \ref{prop: refl}, we have
    \begin{align*}
        (w1.1)+(w1.3)\le  C e^{-c t}\Big(\mathcal{W}_{\norm{.}}(\delta_0,\mu^\eta)+\mathcal{W}_{\norm{.}}(\delta_0,\mu)\Big).
    \end{align*}
    From Lemma \ref{lem: mu bound}, we conclude that $(w1.1)+(w1.3)\le  C^{(1)}C e^{-c t}$.   For the middle term, by using Theorem \ref{Th: strong con}  with initial value $\theta_0=0$, we get
    \begin{align*}
        (w1.2)\le \E[\norm{\theta_t-\zeta_t}]\le C_{\Phi,0,d} e^{8(L+1) t}\eta^{\frac{1}{4}}.
    \end{align*}
    We set $t= -\frac{1}{32(L+1)+4c}\log\eta.$   Hence,
    \begin{align*}
        \mathcal{W}_{\norm{.}}(\mu^\eta,\mu)\le  (w1.1)+(w1.2)+(w1.3)\le C_{\Phi,d} \eta^{c_\Phi},
    \end{align*}
    where $c_\Phi= \frac{c}{32(L+1)+4c}  $ and $C_{\Phi,d}=C_{\Phi,0,d}+C^{(1)}C.$
\end{proof}

\section{Dissipativeness is weaker than Assumption \ref{ass:  conv}}\label{app: assp}

To bring Assumption \ref{ass:  conv} into the context of other analyses of SGLD algorithms, we remark that the dissipativeness assumption assumed in the non-convex analysis of SGLD-type algorithms (see e.g. \citealp{pmlr-v65-raginsky17a, NEURIPS2018_9c19a2aa, pmlr-v161-zou21a}) is weaker than Assumption \ref{ass:  conv}. 
Recall the dissipativeness assumption as the following.
\begin{definition}[Dissipativeness] 
A function $f(\cdot)$ is $(m, b)$-dissipative if for some $m>0$ and $b>0$,
$$\ip{x, \nabla f(x)} \geq m\norm{x}^2 - b,\ \ \ \  \forall x\in \mathbb{R}^d.$$
\end{definition}
Intuitively, dissipativeness means that the function $f(\cdot)$ grows like a quadratic function outside of a ball.
The following lemma shows that Assumption \ref{ass:  conv} implies dissipativeness. The converse implication, however, is incorrect: after proving the lemma, we give an example of a function satisfying the dissipativeness condition, but not Assumption \ref{ass:  conv}. 
\begin{lemma}
    Assume $\{\Phi_i\}_{i\in I}$ satisfy Assumption \ref{ass:  conv} with $(R,K)$. Then there exists a constant $b\ge 0$, such that $\{\Phi_i\}_{i\in I}$ is $(K/2,b)$-dissipative, i.e. for any $i\in I,$
    \begin{align*}
        \ip{x,\nabla \Phi_i(x)}\ge \frac{K}{2}\norm{x}^2-b.
    \end{align*}
\end{lemma}
\begin{proof}
When $\norm{x}\le R$,
\begin{align*}
        \ip{x,\nabla \Phi_i(x)}\ge& -\norm{x}\norm{\nabla \Phi_i(x)}\\
        &\ge \frac{K}{2}\norm{x}^2- \frac{K}{2}R^2-\norm{x}\norm{\nabla \Phi_i(x)}\\
        &\ge \frac{K}{2}\norm{x}^2- \frac{K}{2}R^2-R\sup_{i\in I}\sup_{\norm{x}\le R}\norm{\nabla \Phi_i(x)}.
    \end{align*}

For $\norm{x}\ge R$, by choosing $y=0$ in Assumption \ref{ass:  conv}, we have
\begin{align*}
    \ip{x,\nabla \Phi_i(x)-\nabla \Phi_i(0)}\ge K\norm{x}^2,
\end{align*}
which implies $$\ip{x,\nabla \Phi_i(x)}\ge K\norm{x}^2+ \ip{x,\nabla \Phi_i(0)}.$$ 
By $\e$-Young's inequality, 
$$\ip{x,\nabla \Phi_i(0)}\ge -\norm{x}\norm{\nabla \Phi_i(0)}\ge -\frac{K}{2}\norm{x}^2- \frac{1}{2K}  \sup_{i\in I}\norm{\nabla \Phi_i(0)}^2.$$ 
Hence, for $\norm{x}\ge R$, we have for any $i\in I,$
\begin{align*}
    \ip{x,\nabla \Phi_i(x)}\ge& K\norm{x}^2-\frac{K}{2}\norm{x}^2- \frac{1}{2K}  \sup_{i\in I}\norm{\nabla \Phi_i(0)}^2\\
    \ge& \frac{K}{2}\norm{x}^2- \frac{1}{2K}  \sup_{i\in I}\norm{\nabla \Phi_i(0)}^2.
\end{align*}
Set $b=\max\{\frac{K}{2}R^2+R\sup_{i\in I}\sup_{\norm{x}\le R}\norm{\nabla \Phi_i(x)},\frac{1}{2K}  \sup_{i\in I}\norm{\nabla \Phi_i(0)}^2\},$ we get 
\begin{align*}
        \ip{x,\nabla \Phi_i(x)}\ge \frac{K}{2}\norm{x}^2-b.
    \end{align*}
\end{proof}

\begin{example}
 We let $X := \mathbb{R}$.  We give $\Phi$ through its derivative $\Phi'(x)$. The latter is the odd function defined in the following way, for $0\le x\le 2$, $\Phi'(x)=x.$ In the case $x\ge 2,$ there exist $n\ge 1, $ such that $2^n\le x< 2^{n+1}$, we define:
\begin{equation}\label{worker equation1}
\Phi'(x)= \begin{cases}  2^n, &\text{if  }2^n\le x\le 2^n+\log(n); \\ 
\frac{2^n}{2^n-\log(n)}(x-2^n-\log(n))+2^n, &\text{if  } 2^n+\log(n)< x< 2^{n+1}.\end{cases} 
\end{equation}
We can verify that $x/2\le \Phi'(x)\le x$ for $x\ge 0,$ hence we have $x\Phi'(x)\ge x^2/2$ and $\Phi$ satisfies dissipativeness with $(m, b) = (1/2, 0)$. However, for any $n \in \mathbb{N}$ and $x,y \in [2^n,2^n+\log(n)],$ we have $\Phi'(x)-\Phi'(y)=0.$ Therefore, $\Phi$ does not satisfy Assumption \ref{ass:  conv}.
\end{example}

\end{appendix}

\end{document}